%% file: aaai2019_GE.tex
\def\year{2019}\relax
\documentclass[letterpaper]{article} %DO NOT CHANGE THIS
\usepackage{aaai19}  %Required
\usepackage{times}  %Required
\usepackage{helvet}  %Required
\usepackage{courier}  %Required
\usepackage{url}  %Required
\usepackage{graphicx}  %Required

\usepackage{amsmath,amssymb,amsfonts,amstext,amsthm,bm}
\usepackage[utf8]{inputenc} % allow utf-8 input
\usepackage[T1]{fontenc}    % use 8-bit T1 fonts
\usepackage{hyperref}       % hyperlinks
\usepackage{url}            % simple URL typesetting
\usepackage{booktabs}       % professional-quality tables
\usepackage{nicefrac}       % compact symbols for 1/2, etc.
\usepackage{microtype}      % microtypography
\usepackage{cleveref}
\usepackage{dsfont}
\usepackage{enumitem,subfigure,epstopdf}
\usepackage{thm-restate}
\usepackage{color,float}

\newtheorem{definition}{Definition}
\newtheorem{proposition}{Proposition}
\newtheorem{lemma}{Lemma}
\newtheorem{thm}{Theorem}

\newtheorem{assum}{Assumption}

\makeatletter
\newcommand*\bigcdot{\mathpalette\bigcdot@{.5}}
\newcommand*\bigcdot@[2]{\mathbin{\vcenter{\hbox{\scalebox{#2}{$\m@th#1\bullet$}}}}}

\makeatother

\newcommand{\Ebb}{\mathbb{E}}

\newcommand{\gb}{\mathbf{g}}

\newcommand{\zb}{\mathbf{z}}
\newcommand{\wb}{\mathbf{w}}
\newcommand{\ub}{\mathbf{u}}
\newcommand{\vb}{\mathbf{v}}
\newcommand{\RR}{\mathds{R}}
\newcommand{\Oc}{O}

\newcommand{\prox}[1]{\mathrm{prox}_{#1}}

\newcommand{\zero}{\mathbf{0}}

\newcommand{\inner}[2]{\langle #1, #2 \rangle}
\newcommand{\norm}[1]{\| #1 \|}
\newcommand{\norml}[1]{\left\| #1 \right\|}

\begin{document}
% The file aaai.sty is the style file for AAAI Press 
% proceedings, working notes, and technical reports.
%
\title{Understanding Generalization Error of SGD in Nonconvex Optimization}
%\name{Yi Zhou$^1$, Huishuai Zhang$^2$ and Yingbin Liang$^3$ \thanks{$^1$ Corresponding author. Email: yi.zhou610@duke.edu}}
%\address{$^1$ Department of ECE, Duke University \\ $^2$ Microsoft Research, Asia \\ $^3$ Department of ECE, The Ohio State University}
\author{Yi Zhou\footnote{Corresponding author}\\ Duke University \\ Email: yi.zhou610@duke.edu \And Yingbin Liang\\
	The Ohio State University\\ Email: liang.889@osu.edu \And Huishuai Zhang \\ Microsoft Research, Asia \\ Email: huzhang@microsoft.com}
\maketitle
\begin{abstract}
The success of deep learning has led to a rising interest in the generalization property of the stochastic gradient descent (SGD) method, and stability is one popular approach to study it. Existing generalization bounds based on stability do not incorporate the interplay between the optimization of SGD and the underlying data distribution, and hence cannot even capture the effect of randomized labels on the generalization performance. 
In this paper, we establish generalization error bounds for SGD by characterizing the corresponding stability in terms of the on-average variance of the stochastic gradients. Such characterizations lead to improved bounds on the generalization error of SGD and experimentally explain the effect of the random labels on the generalization performance. 
We also study the regularized risk minimization problem with strongly convex regularizers, and obtain improved generalization error bounds for the proximal SGD. 
\end{abstract}

\input{introduction.tex}

\input{stability.tex}

\input{sgd_generalization.tex}

\input{composite.tex}

\section{Conclusion}
In this paper, we develop the generalization error bound of SGD with probabilistic guarantee for nonconvex optimization. 
We obtain the improved bounds based on the variance of the stochastic gradients 
by exploiting the optimization path of SGD. Our generalization bound is consistent with the effect of random labels on the generalization error that observed in practical experiments. We further show that strongly convex regularizers can significantly improve the probabilistic concentration bounds for the generalization error from the sub-linear rate to the exponential rate. Our study demonstrates that the geometric structure of the problem can be an important factor in improving the generalization performance of algorithms. Thus, it is of interest to explore the generalization error under various geometric conditions of the objective function in the future work.

\bibliographystyle{aaai}
\bibliography{ref}

\onecolumn
\input{supplementary.tex}

\end{document}

%% file: introduction.tex
%\vspace{-.5cm}
\section{Introduction}
Many machine learning applications can be formulated as risk minimization problems, in which each data sample $\zb\in \RR^p$ is assumed to be generated by an underlying multivariate distribution $\mathcal{D}$. The loss function $\ell(\cdot; \zb) : \RR^d \to \RR$ measures the performance on the sample $\zb$ and its form depends on specific applications, e.g., square loss for linear regression problems, logistic loss for classification problems and cross entropy loss for training deep neural networks, etc. The goal is to solve the following population risk  minimization (PRM) problem over a certain parameter space $\Omega \subset \RR^d$.
\begin{align*}
\min_{\wb \in \Omega} ~ f(\wb):= \Ebb_{\zb \sim \mathcal{D}} ~\ell(\wb; \zb). \tag{PRM}
\end{align*}  
Directly solving the PRM can be difficult in practice, as either the distribution $\mathcal{D}$ is unknown or evaluation of the expectation of the loss function induces high computational cost. To avoid such difficulties, one usually samples a set of $n$ data samples $S := \{\zb_1, \ldots, \zb_n \}$ from the distribution $\mathcal{D}$, and instead solves the following empirical risk minimization (ERM) problem.
\begin{align*}
\min_{\wb \in \Omega} ~ f_S(\wb):= \frac{1}{n} \sum_{k=1}^{n} \ell(\wb; \zb_k). \tag{ERM}
\end{align*}  
The ERM serves as an approximation of the PRM with finite samples. In particular, when the number $n$ of data samples is large, one wishes that the solution $\wb_S$ found by optimizing the ERM with the data set $S$ has a good generalization performance, i.e., it also induces a small loss on the population risk. The gap between these two risk functions is referred to as the {\em generalization error} at $\wb_S$, and is formally written as 
\begin{align}
	\text{(generalization error)}:= |f_S (\wb_S) - f(\wb_S)|. \label{eq: GE}
\end{align}

Various theoretical frameworks have been established to study the generalization error from different aspects (see related work for references).
%e.g., VC dimension \cite{Vapnik_1995}, large margin classifier \cite{Bartlett_2004}, PAC-Bayesian theory \cite{McAllester_1999}, algorithm robustness \cite{Xu_2012} and algorithm stability \cite{Bousquet_2002}, etc. 
This paper adopts the stability framework \cite{Bousquet_2002,Elisseeff_2005}, which has been applied to study the generalization property of the output produced by learning algorithms. More specifically, for a particular learning algorithm $\mathcal{A}$, its stability corresponds to how stable the output of the algorithm is with regard to the variations in the data set. As an example, consider two data sets $S$ and $\overline{S}$ that differ at one data sample, and denote $\wb_{S}$ and $\wb_{\overline{S}}$ as the outputs of algorithm $\mathcal{A}$ when applied to solve the ERM with the data sets $S$ and $\overline{S}$, respectively. Then, the stability of the algorithm measures the gap between the output function values of the algorithm on the perturbed data sets.

%The stability framework has been applied to study the generalization property of the output produced by learning algorithms \cite{Bousquet_2002,Elisseeff_2005}.
Recently, the stability framework has been further developed to study the generalization performance of the output produced by the stochastic gradient descent (SGD) method from various theoretical aspects \cite{Hardt_2016,Charles_2017,Mou_2017,Yin_2017,Kuzborskij_2017}. These studies showed that the output of SGD can achieve a vanishing generalization error after multiple passes over the data set as the sample size $n\to \infty$. These results provide theoretical justifications in part to the success of SGD on training complex objectives such as deep neural networks. 

However, as pointed out in \cite{Zhang2017}, these bounds do not explain some experimental observations, e.g., they do not capture the change of the generalization performance as the fraction of random labels in training data changes. Thus, the aim of this paper is to develop better generalization bounds that incorporate both the optimization information of SGD and the underlying data distribution, so that they can explain experimental  observations. We summarize our contributions as follows.
%Generally, two metrics are typically taken for measuring the generalization error. One is the generalization error {\em in expectation} given by
%\begin{align}
%	\Ebb_{\mathcal{A}, S} \left(f_S (\wb_{T,S}) - f (\wb_{T,S})\right), \label{eq: expect_gene}
%\end{align} 
%where $\wb_{T, S}$ corresponds to the output of SGD at the $T$-th iteration when applied to solve the ERM with the data set $S$, and the expectation is taken over the randomness of both the algorithm and the draw of the data. The second one is the generalization error bound with probabilistic guarantee, i.e., for any $\epsilon > 0$ the quantity
%\begin{align}
%	\mathbb{P}(\left|f_S (\wb_{T,S}) - f (\wb_{T,S})\right|< \epsilon)  \label{eq: prob_gene}
%\end{align} 
%converges to one as $n \to \infty$. Compared to the {\em in expectation} guarantee, the probabilistic guarantee is a stronger metric that guarantees small generalization error with high probability.

%to enrich the understanding of the generalization performance of the SGD in the algorithm stability framework. In particular, we want to incorporate optimization properties and regularization techniques into the analysis of the algorithm stability, and establish improved generalization bounds with in probability guarantees. 

\subsection*{Our Contributions}

For smooth nonconvex optimization problems, we propose a new analysis of the on-average stability of SGD that exploits the optimization properties as well as the underlying data distribution. Specifically, via upper-bounding the on-average stability of SGD, we provide a novel generalization error bound, which improves upon the existing bounds by incorporating the on-average {\em variance} of the stochastic gradient. We further corroborate the connection of our bound to the generalization performance of the recent experiments in \cite{Zhang2017}, which were not explained by the existing bounds of the same type. In specific, our experiments demonstrate that the obtained generalization bound captures how the generalization error changes with the fraction of random labels via the on-average {\em variance} of SGD. Furthermore, our bound holds under probabilistic guarantee, which is statistically stronger than the bounds in expectation provided in, e.g., \cite{Hardt_2016,Kuzborskij_2017}. Then, we study nonconvex optimization under gradient dominance condition, and show that the corresponding generalization bound for SGD can be improved by its fast convergence rate.

%improves the existing bounds of the same type (e.g., \cite{Hardt_2016}) by its nature and is also statistically stronger with probabilistic guarantee. 
% is consistent with the effect of the random labels on the generalization performance. Such result further leads to improved generalization error bound with probabilistic guarantee compared to that in \cite{Hardt_2016}. 

%showed that the on-average stability of SGD is bounded by the on-average variance of the stochastic gradient and the population risk at initialization. Such result further leads to improved generalization error bound with probabilistic guarantee compared to that in \cite{Hardt_2016}. Moreover, our experimental results showed that the obtained generalization bound is consistent with the effect of the random labels on the generalization performance.

%For nonconvex functions that satisfy the gradient dominance condition, \cite{Charles_2017} analyzed the generalization error  further under a quadratic growth condition, and the analysis also required SGD to converge to a global minimizer. In contrast, this paper does not require additional conditions other than the gradient dominance condition. We show that the gradient dominance condition does improve the generalization error bound compared to general nonconvex functions. 

We further consider nonconvex problems with strongly convex regularizers, and study the role that the regularization plays in characterizing the generalization error bound of the proximal SGD. In specific, our generalization bound shows that strongly convex regularizers substantially improve the generalization bound of SGD for {\em nonconvex} loss functions to be as good as the strongly convex loss function. Furthermore, the uniform stability of SGD under a strongly convex regularizer yields a generalization bound for {\em nonconvex} problems with exponential concentration in probability. We also provide some experimental observations to support our result.

%Although \cite{Mou_2017} also studied regularized nonconvex problems, they only considered a particular choice of strongly convex regularizer and the high-probability guarantee is only with respect to the data randomness. In comparison, we consider all strongly convex regularizers and our probabilistic guarantee is with respect to both the algorithm and data randomness. 

%Specifically, we characterize the generalization error bounds based on the on-average variance of the stochastic gradients. Our results show that strongly convex regularizers can substantially improve the generalization error bounds of SGD in nonconvex optimization. In particular, a general nonconvex loss function under a strongly convex regularizer can achieve the same order-level generalization error bound of a strongly convex loss function as characterized in \cite{London_2017}.

\subsection*{Related Works}
The stability approach was initially proposed by \cite{Bousquet_2002} to study the generalization error, where various notions of stability were introduced to provide bounds on the generalization error with probabilistic guarantee. \cite{Elisseeff_2005} further extended the stability framework to characterize the generalization error of randomized learning algorithms. \cite{Shalev_2010} developed various properties of stability on learning problems. In \cite{Hardt_2016}, the authors first applied the stability framework to study the expected generalization error for SGD, and \cite{Kuzborskij_2017} further provided a data dependent generalization error bound. In \cite{Mou_2017}, the authors studied the generalization error of SGD with additive Gaussian noise. \cite{Yin_2017} studied the role that gradient diversity plays in characterizing the expected generalization error of SGD. All these works studied the expected generalization error of SGD. In \cite{Charles_2017}, the authors studied the generalization error of several first-order algorithms for loss functions satisfying the gradient dominance and the quadratic growth conditions. \cite{Poggio_2011} studied the stability of online learning algorithms. This paper improves the existing bounds by incorporating the on-average variance of SGD into the generalization error bound and further corroborates its connection to the generalization performance via experiments. More detailed comparison with the existing bounds are given after the presentation of main results.

The PAC Bayesian theory \cite{Valiant_1984,McAllester_1999} is another popular framework for studying the generalization error in machine learning. It was recently used to develop bounds on the generalization error of SGD \cite{London_2017,Mou_2017}. Specifically, \cite{Mou_2017} applied the PAC Bayesian theory to study the generalization error of SGD with additive Gaussian noise. \cite{London_2017} combined the stability framework with the PAC Bayesian theory and provided bound on the generalization error with probabilistic guarantee of SGD. The bound incorporates the divergence between the prior distribution and the posterior distribution of the parameters. 

Recently, \cite{Russo_2016,Xu_2017} applied information-theoretic tools to characterize the generalization capability of learning algorithms, and \cite{Pensia_2018} further extended the framework to study the generalization error of various first-order algorithms with noisy updates. 
Other approaches were also developed for characterizing the generalization error as well as the estimation error, which include, for example, the algorithm robustness framework \cite{Xu_2012,Zahavy_2017}, large margin theory \cite{Bartlett_2017,Neyshabur_2017,Sokolic_2017} and the classical VC theory \cite{Vapnik_1995,Vapnik_1998}. Also, some methods have been developed to study excessive risk of the output for a learning algorithm, which include the robust stochastic approach \cite{Nemirovski_2009}, the sample average approximation approach \cite{Shapiro_2005,Lin_2017}, etc.

%% file: stability.tex
\section{Preliminary and On-Average Stability}\label{sec:formulation}
Consider applying SGD to solve the empirical risk minimization (ERM) with a particular data set $S$. In particular, at each iteration $t$, the algorithm samples one data sample from the data set $S$ uniformly at random. Denote the index of the sampled data sample at the $t$-th iteration as ${\xi}_t$. Then, with a stepsize sequence $\{\alpha_t\}_t$ and a fixed initialization $\wb_{0} \in \RR^d$, the update rule of SGD can be written as, for $t = 0, \ldots, T-1$,
\begin{align}
\wb_{t+1} = \wb_{t} -  \alpha_t \nabla \ell (\wb_{t}; \zb_{\xi_t}). \tag{SGD}
\end{align}  
Throughout the paper, we denote the iterate sequence along the optimization path as $\{\wb_{t, S}\}_t$, where $S$ in the subscript indicates that the sequence is generated by the algorithm using the data set $S$. The stepsize sequence $\{\alpha_t\}_t$ is a decreasing and positive sequence, and typical choices for SGD are $\frac{1}{t}, \frac{1}{t\log t}$ \cite{Bottou_2010}, which we adopt in our study.

Clearly, the output $\wb_{T, S}$ is determined by the data set $S$ and the sample path $\bm{\xi}:=\{{\xi_1}, \ldots, {\xi_{T-1}} \}$ of SGD. 
We are interested in the generalization error of the $T$-th output generated by SGD, i.e., $|f_S (\wb_{T,S}) - f (\wb_{T,S})|$, and we adopt the following standard assumptions \cite{Hardt_2016,Kuzborskij_2017} on the loss function $\ell$ in our study throughout the paper.
\begin{assum}\label{assum: loss}
	For all $\zb \sim \mathcal{D}$, the loss function satisfies:
	\begin{enumerate}[leftmargin=*,topsep=0pt,noitemsep]
	\item Function $\ell(\cdot~;\zb)$ is continuously differentiable;
	\item Function $\ell(\cdot~;\zb)$ is nonnegative and $\sigma$-Lipschitz, and $|\ell(\cdot~;\zb)|$ is uniformly bounded by $M$; 
	\item The gradient $\nabla \ell(\cdot~;\zb)$ is $L$-Lipschitz, and $\norm{\nabla \ell(\cdot~;\zb)}$ is uniformly bounded by $\sigma$, where $\|\cdot\|$ denotes the $\ell_2$ norm.
	\end{enumerate}
\end{assum}
The generalization error of SGD can be viewed as a nonnegative random variable whose randomnesses are due to the draw of the data set $S$ and the sample path $\bm{\xi}$ of the algorithm. In particular, the mean square generalization error has been studied in \cite{Elisseeff_2005} for general randomized learning algorithms. Specifically, an application of [Lemma 11, \cite{Elisseeff_2005}] to SGD under \Cref{assum: loss} yields the following result. 
Throughout the paper, we denote $\overline{S}$ as the data set that replaces one data sample of $S$ with an i.i.d copy generated from the distribution $\mathcal{D}$ and denote $\wb_{T, \overline{S}}$ as the output of SGD for solving the ERM with the data set $\overline{S}$.
\begin{proposition}\label{thm: stability}
	Let \Cref{assum: loss} hold. Apply the SGD with the same sample path $\bm{\xi}$ to solve the ERM with the data sets $S$ and $\overline{S}$, respectively. Then, the mean square generalization error of SGD satisfies
	\begin{align}
		&\Ebb [|f_S (\wb_{T,S}) - f (\wb_{T,S})|^2] \le \frac{2M^2}{n}  + 12M\sigma \Ebb [\delta_{T,S,\overline{S}}], \nonumber
		%&\Ebb \norm{\nabla f_S (\wb_{T,S}) - \nabla f (\wb_{T,S})}^2 \le \tfrac{2\sigma^2}{n}  + 6L\sigma \Ebb \left[\delta_{T,S,\overline{S}}\right]. 
	%\nonumber
	\end{align}
	where $\delta_{T,S,\overline{S}} := \norm{\wb_{T, S} - \wb_{T, \overline{S}}}$ and the expectation is taken over the random variables $\overline{S}, S$ and $\bm{\xi}$. 
\end{proposition}

\Cref{thm: stability} links the mean square generalization error of SGD to the quantity $\Ebb_{\bm{\xi}, S, \overline{S}} [\delta_{T,S,\overline{S}}]$. Intuitively, $\delta_{T,S,\overline{S}}$ captures the variation of the algorithm output with regard to the variation of the dataset. Hence, its expectation can be understood as the {\em on-average stability} of the iterates generated by SGD. We note that similar notions of stabilities were proposed in \cite{Kuzborskij_2017,Shalev_2010,Elisseeff_2005}, which are based on the variation of the function values at the output instead.

%% file: sgd_generalization.tex
%\section{On-Average Stability, Optimization Path and Generalization Error}
\section{Generalization Bound for SGD in Nonconvex Optimization}\label{sec:gene}
In this section, we develop the generalization error of SGD by characterizing the corresponding on-average stability of the algorithm. 

An intrinsic quantity that affects the optimization path of SGD is the variance of the stochastic gradients. To capture the impact of the variance of the stochastic gradients, we adopt the following standard assumption from the stochastic optimization theory \cite{Bottou_2010,Nemirovski_2009,Ghadimi_2016}.
\begin{assum}\label{assum: var}
	For any fixed training set $S$ and any $\xi$ that is generated uniformly from $\{1, \ldots, n\}$ at random, there exists a constant $\nu_S > 0$ such that for all $\wb \in \Omega$ one has 
	\begin{align}\label{eq:gradvar}
	\Ebb_{\xi} \Big\|\nabla \ell (\wb; \zb_{\xi}) - \frac{1}{n} \sum_{k=1}^n \nabla \ell (\wb; \zb_{k})\Big\|^2 \le \nu_S^2.
	\end{align}
\end{assum}
\Cref{assum: var} essentially bounds the variance of the stochastic gradients for the particular data set $S$. The variance $\nu_S^2$ of the stochastic gradient is typically much smaller than the uniform upper bound $\sigma$ in \Cref{assum: loss} for the norm of the stochastic gradient, e.g., normal random variable has unit variance and is unbounded, and hence may provide a tighter bound on the generalization error. 

%In our case of mini-batch SGD, we need to consider the on-average variance of the mini-batch stochastic gradients, i.e.,
%\begin{align}
%&\var_{S,\bm{\xi}}(\grad_{t, S}) := \nonumber\\
%&\Ebb_{S, \bm{\xi}} \norml{\frac{1}{m} \sum_{k=1}^m \nabla \ell (\wb_{t, S}; \zb_{\xi(t,k)}) - \frac{1}{n} \sum_{k=1}^n \nabla \ell (\wb_{t, S}; \zb_{k})}^2. \nonumber
%\end{align}
%We note that the above variance is further averaged over the data set $S$. Intuitively, it can be understood as the on-average variance of the stochastic gradients along the iteration path with different draws of the data set.
%Under \Cref{assum: var} and the sampling without replacement rule of the mini-batch SGD, standard statistics of sampling without replacement implies that
%\begin{align}
%	\var_{\bm{\xi}, S}(\grad_{t}) \le \frac{\Ebb_S[\nu_S^2]}{m}\left(1 - \frac{m(m-1)}{n(n-1)}\right) \le \frac{\Ebb_S[\nu_S^2]}{m}. \nonumber
%\end{align}

Based on \Cref{assum: var} and \Cref{thm: stability}, we obtain the following generalization bound of SGD by exploring its optimization path to study the corresponding stability.
\begin{thm}(Bound with Probabilistic Guarantee) \label{coro: prob_gene}
	Suppose $\ell$ is nonconvex. Let Assumptions \ref{assum: loss} and \ref{assum: var} hold. Apply the SGD to solve the ERM with the data set $S$, and choose the step size $\alpha_t = \frac{c}{(t+2)\log (t+2)}$ with $0<c<\frac{1}{L}$. Then, for any $\delta>0$, with probability at least $1-\delta$, we have
	\begin{align}
	&|f_S (\wb_{T,S}) - f (\wb_{T,S})| \nonumber\\
	&\le  \sqrt{\frac{1}{n\delta} \Big(2M^2+ 24M\sigma c\sqrt{2Lf(\wb_{0}) + \frac{1}{2}\Ebb_S[\nu_S^2]}  \log T\Big)}. \nonumber
	\end{align}
\end{thm}
\begin{proof}[Outline of the Proof of \Cref{coro: prob_gene}] 
	We provide an outline of the proof here, and relegate the detailed proof in the supplementary materials. 
	
	The central idea is to bound the on-average stability $\Ebb_{S, \overline{S}, \bm{\xi}} [\delta_{T, S, \overline{S}}]$ of the iterates in \Cref{thm: stability}. Hence, suppose we apply SGD with the same sample path $\bm{\xi}$ to solve the ERM with the data sets $S$ and $\overline{S}$, respectively. We first obtain the following recursive property of the on-average iterate stability (\Cref{lemma: GE_general} in the appendix):
	\begin{align}
	\Ebb_{S, \overline{S}, \bm{\xi}} [\delta_{t+1, S, \overline{S}}] 
	&\le (1 + \alpha_t L) \Ebb_{S, \overline{S}, \bm{\xi}} [\delta_{t, S, \overline{S}}]  \nonumber\\
	&\quad+ \frac{2\alpha_t}{n}\Ebb_{S, \bm{\xi}} \left[\norml{\nabla \ell (\wb_{t, S}; \zb_{1})}\right]. \label{eq: recursive}
	\end{align} 
	
	We then further derive the following bound on $\Ebb_{S, \bm{\xi}} \left[\norm{\nabla \ell (\wb_{t, S}; \zb_{1})}\right] $ by exploiting the optimization path of SGD (\Cref{lemma: grad_bound} in the appendix):
	\begin{align}\label{eq:gradexp}
	&\Ebb_{\bm{\xi}, S} \left[\norm{\nabla \ell (\wb_{t, S}; \zb_{1})}\right] \le  \sqrt{2Lf(\wb_{0}) + \frac{1}{2}\Ebb_{S}[\nu_S^2]}. 
	\end{align}
	
	Substituting \cref{eq:gradexp} into \cref{eq: recursive} and telescoping, we obtain an upper bound on $\Ebb_{S, \overline{S}, \bm{\xi}} [\delta_{T, S, \overline{S}}] $. Further substituting such a bound into \Cref{thm: stability}, we obtain an upper bound on the second moment of the generalization error. Then, the result in \Cref{coro: prob_gene} follows from the Chebyshev's inequality.
\end{proof} 

The proof of \Cref{coro: prob_gene} is to characterize the on-average stability of SGD, and it explores the optimization path by applying the technical tools developed in stochastic optimization theory. Comparing to the generalization bound developed in \cite{Hardt_2016} that characterizes the expected generalization error based on the uniform stability $\sup_{S, \overline{S}}\Ebb_{\bm{\xi}} [\delta_{T,S,\overline{S}}]$, our generalization bound in \Cref{coro: prob_gene} provides a probabilistic guarantee, and is based on the more relaxed on-average stability $\Ebb_{S, \overline{S}}\Ebb_{\bm{\xi}} [\delta_{T,S,\overline{S}}]$ which yields a tighter bound. Intuitively, the on-average variance term ${\Ebb_{S}[\nu_S^2]}$ in our bound measures the `stability' of the stochastic gradients over all realizations of the dataset $S$. If such on-average variance of SGD is large, then the optimization paths of SGD on two slightly different datasets are diverse from each other, leading to the bad stability of SGD and in turn yielding a high generalization error. 
We note that \cite{Kuzborskij_2017} also exploited the optimization path to characterize the {\em expected} generalization error of SGD. However, their analysis assumes that the iterate $\wb_{t, S}$ is independent of $\zb_{\xi_{t+1}}$, which may not hold after multiple passes over the data samples. Also, their result does not capture the on-average variance of the stochastic gradients. 

We next explain how our generalization bound can explain observations in experiments. The generalization bound in \Cref{coro: prob_gene} depends on the on-average variance ${\Ebb_{S}[\nu_S^2]}$ of the stochastic gradients, which incorporates the underlying data distribution and can capture its effect on the generalization performance. We conduct several experiments to demonstrate that the on-average variance of the SGD does capture the generalization performance. For example, it has been observed that a dataset with true labels leads to good generalization performance whereas a dataset with random labels leads to bad generalization performance \cite{Zhang2017}. 
Following this observation, we perform three experiments: solving a logistic regression with the a9a dataset \cite{Chang_2011}, training a three-layer ReLU neural network with the MNIST dataset \cite{Lecun_1998} and training a Resnet-18 \cite{He2016} with the CIFAR10 dataset \cite{Krizhevsky09}. In specific, we vary fraction of random labels (i.e., vary the probability of replacing true labels to randomly selected labels) in the datasets and 
%various settings of the probability of replacing the true labels with a random label and 
evaluate the on-average variance of SGD for the last multiple iterations of the training process. For neural network experiments, we terminate the training process when the training error is below $0.2\%$ for all settings of random label probability. 
Also, as the on-average variance is averaged over the data distribution, we adopt the corresponding sample mean over the random draw of the training dataset as an estimation. \Cref{fig: 1} shows our experimental results. For all three experiments with very different objective functions, it can be seen that the on-average variance consistently becomes larger as the fraction of random labels increases (i.e., the generalization error increases). Thus, our empirical study establishes an affirmative connection between the on-average variance (captured in our generalization bound) and the generalization performance in the experiments.
%Intuitively, the direction of the gradients of the data samples become more diverse with more random labels, and hence yields a larger sample variance.
\begin{figure*}[th]
	\centering 
	\subfigure{
		\includegraphics[width=2in]{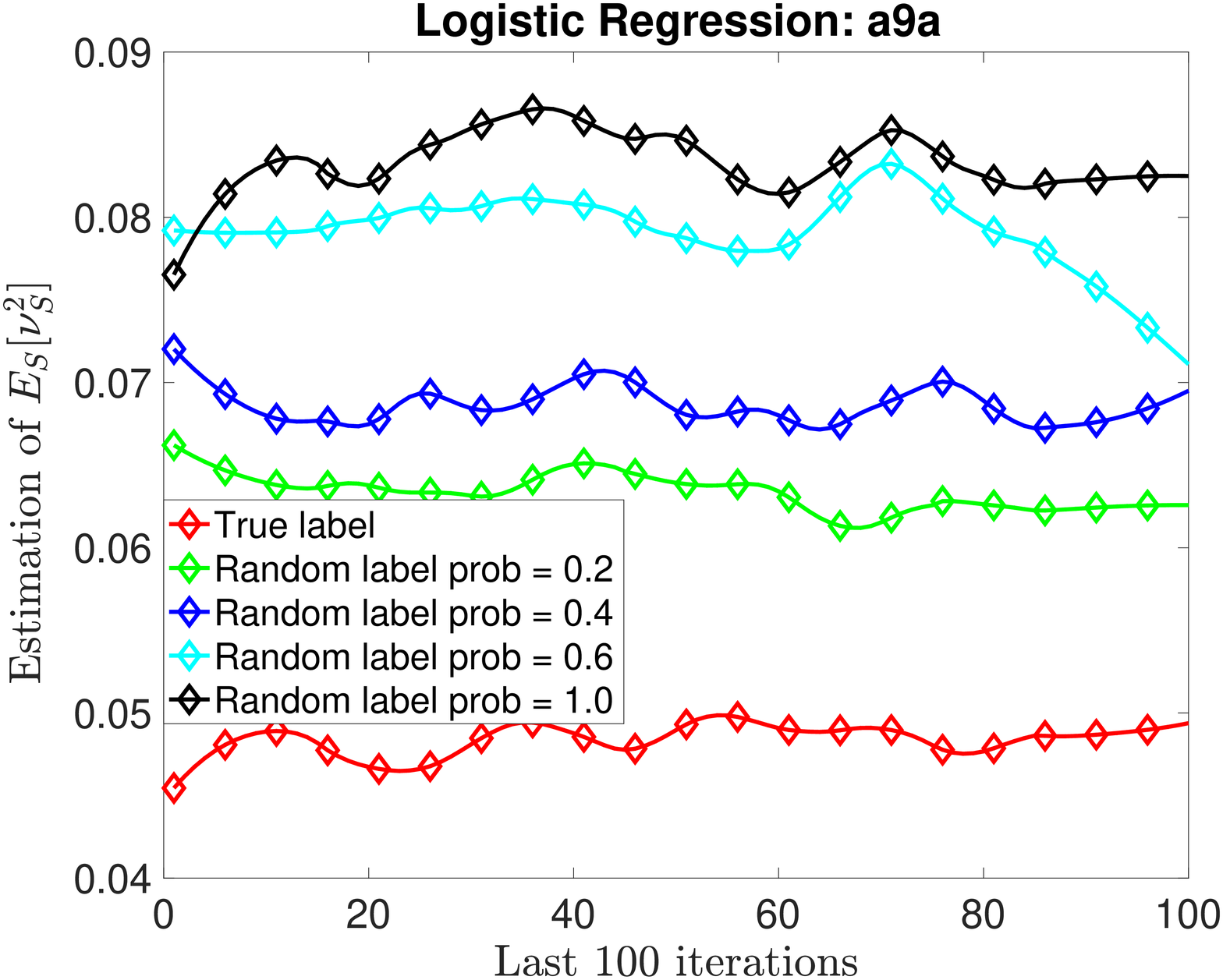}
	}
	\subfigure{
		\includegraphics[width=2in]{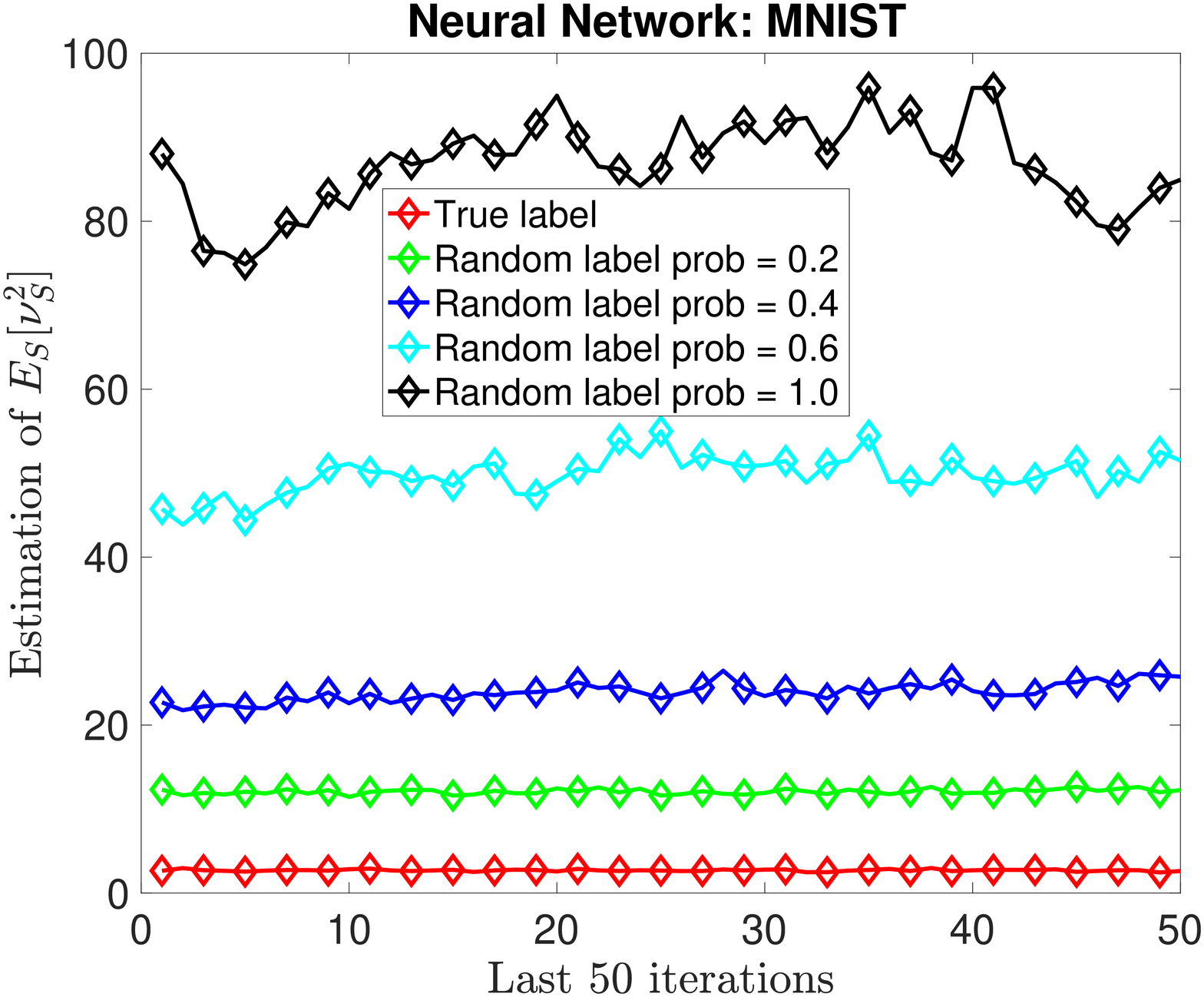}
	}
\subfigure{
	\includegraphics[width=2in]{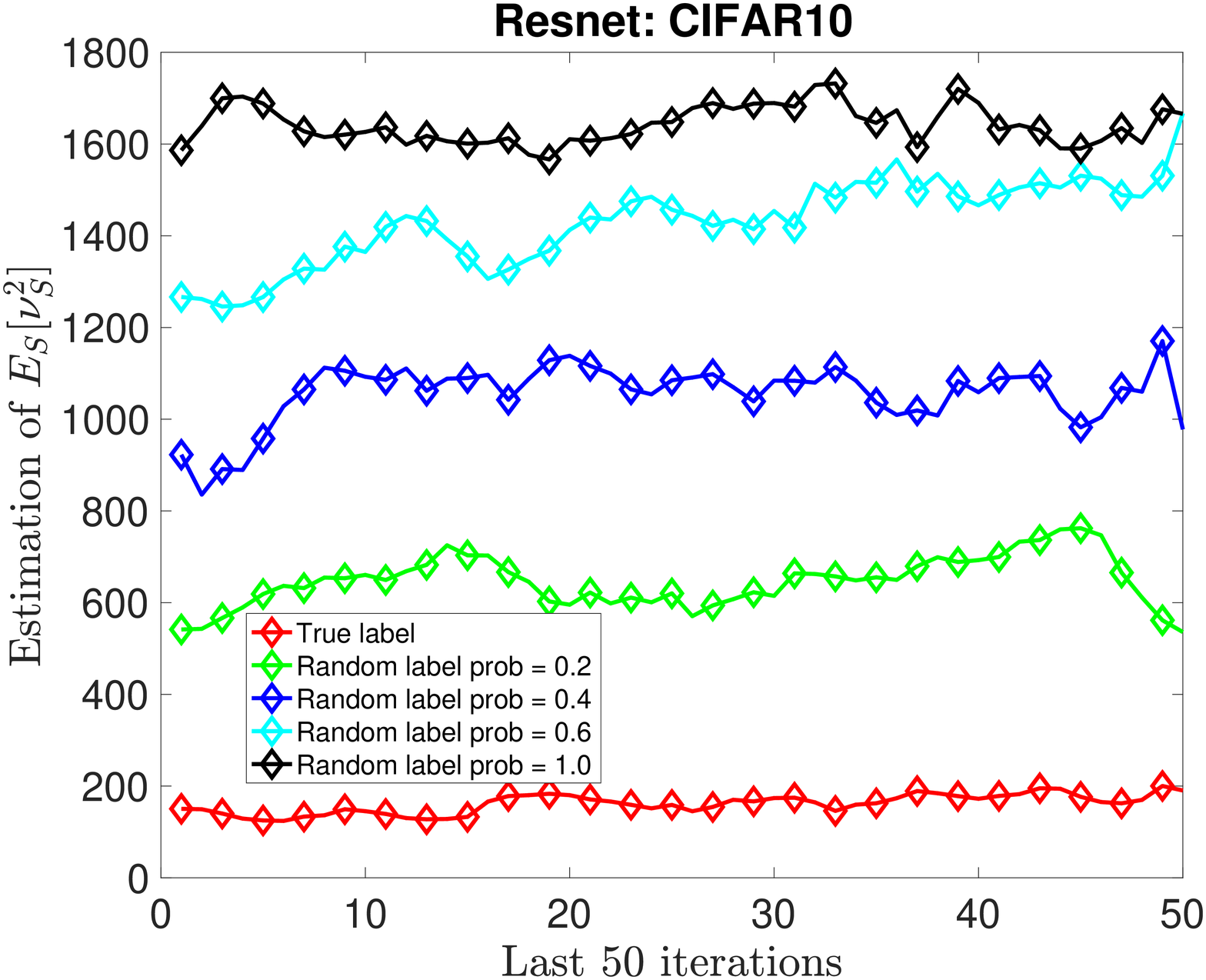}
}
	\caption{On-average variance of SGD v.s. random label probability.} 
	\label{fig: 1}
	\vspace{-0.5cm}
\end{figure*}

\section{Generalization Bound for SGD under Gradient Dominant Condition}
In this section, we consider nonconvex loss functions with the empirical risk function $f_S$ further satisfying the following gradient dominance condition.
\begin{definition}
	Denote $f^* := \inf_{\wb \in \Omega} f(\wb)$. Then, the function $f$ is said to be $\gamma$-gradient dominant for $\gamma>0$ if 
	\begin{align}
	f(\wb) - f^* \le \frac{1}{2\gamma} \norm{\nabla f(\wb)}^2, ~\forall \wb\in \Omega.
	\end{align}
\end{definition}
The gradient dominance condition (also referred to as Polyak-{\L}ojasiewicz condition \cite{Polyak_1963,Lojasiewicz_1963}) guarantees a linear convergence of the function value sequence generated by gradient-based first-order methods \cite{Karimi_2016}. It is a condition that is much weaker than the strong convexity, and many nonconvex machine learning problems satisfy this condition around the global minimizers \cite{Li_2016,Zhou_2016}. 

The gradient dominance condition helps to improve the bound on the on-average stochastic gradient norm $\Ebb_{\bm{\xi}, S} \left[\norm{\nabla \ell (\wb_{t, S}; \zb_{1})}\right]$ (see \Cref{lemma: grad_bound_GD} in the appendix), which is given by
\begin{align}
	\Ebb_{\bm{\xi}, S}& \left[\norm{\nabla \ell (\wb_{t, S}; \zb_{1})}\right] \nonumber\\
	&\le \sqrt{2L \Ebb_{S} [f_S^*] + \frac{1}{t} \bigg(2Lf(\wb_{0}) +  \Ebb_S[\nu_S^2]  \bigg)}. \label{eq:gradexp_gd}
	\end{align}
Compared to \cref{eq:gradexp} for general nonconvex functions, the above bound is further improved by a factor of $\frac{1}{t}$. This is because SGD converges sub-linearly to the optimum function value $f_S^*$ under the gradient dominance condition, and $\frac{1}{t}$ is essentially the convergence rate of SGD. In particular, for sufficiently large $t$, the on-average stochastic gradient norm is essentially bounded by $\sqrt{2L \Ebb_{S} [f_S^*]}$, which is much smaller then the bound in \cref{eq:gradexp}. With the bound in \cref{eq:gradexp_gd}, we obtain the following theorem.
\begin{thm} (Mean Square Bound) \label{thm: gene_gd}
Suppose $\ell$ is nonconvex, and $f_S$ is $\gamma$-gradient dominant ($\gamma<L$). Let Assumptions \ref{assum: loss} and \ref{assum: var} hold. Apply the SGD to solve the ERM with the data set $S$ and choose $\alpha_t = \frac{c}{(t+2)\log (t+2)}$ with $0<c<\min\{\frac{1}{L}, \frac{1}{2\gamma} \} $. Then, the following bound holds.
\begin{align}
&\Ebb_{\bm{\xi}, S} [|f_S (\wb_{T,S}) - f (\wb_{T,S})|^2]\le \nonumber \\
&\!\frac{2M^2}{n} \!+\! \frac{24M\sigma c}{n} \bigg(\!\!\!\sqrt{\!2L \Ebb_{S} [f_S^*]}\log T \!+\! \sqrt{\!2Lf(\!\wb_{0}\!) \!+\!  2\Ebb_S[\nu_S^2]}\!\bigg). \nonumber	
\end{align} 
\end{thm}
%Note that we used a slightly smaller stepsize under the gradient dominance condition in \Cref{thm: gene_gd}. 

The above bound for the mean square generalization error under gradient dominance condition improves that for general nonconvex functions in \Cref{coro: prob_gene}, as the dominant term (i.e., $\log T$-dependent term) has coefficient $\sqrt{2L\Ebb_S[f_S^*]}$, which is much smaller than the term $\sqrt{2L f(\wb_{0}) +  \frac{1}{2}\Ebb_S[\nu_S^2]}$ in the bound of \Cref{coro: prob_gene}. As an intuitively understanding, the on-average variance of the SGD is further reduced by its fast convergence rate $\frac{1}{t}$ under the gradient dominance condition. This results in a more stable on-average iterate stability which in turn improves the mean square generalization error.  
We note that \cite{Charles_2017} also studied the generalization error of SGD for loss functions satisfying both the gradient dominance condition and an additional quadratic growth condition. They also assumed that the algorithm converges to a global minimizer point, which may not always hold for noisy algorithms like SGD. 

%Next, we combine \Cref{thm: stability} with \Cref{thm: gene} and further obtain the following high probability guarantees for the generalization errors of the output generated by the mini-batch SGD. For simplicity of presentation, we have absorbed all universal constants, $f(\wb_{0})$ and $\Ebb_S [f_S^*]$ into the $\Oc(\cdot)$ notation, as they are deterministic quantities under the distribution $\mathcal{D}$.

\Cref{thm: gene_gd} directly implies the following {\em probabilistic} guarantee for the generalization error of SGD.
\begin{thm}\label{coro: prob_gene1} (Bound with Probabilistic Guarantee)
Suppose $\ell$ is nonconvex, and $f_S$ is $\gamma$-gradient dominant ($\gamma<L$). Let Assumptions \ref{assum: loss} and \ref{assum: var} hold. Apply the SGD to solve the ERM with the data set $S$, and choose $\alpha_t = \frac{c}{(t+2)\log(t+2)}$ with $0<c<\min\{\frac{1}{L}, \frac{1}{2\gamma} \}$. Then, for any $\delta > 0$, with probability at least $1-\delta$, we have
\begin{align}
&|f_S (\wb_{T,S}) - f (\wb_{T,S})|  \le\nonumber \\
& \!\!\sqrt{ \!\tfrac{2M^2}{n\delta} \!+\! \tfrac{24M\sigma c}{n\delta} \bigg( \!\!\!\sqrt{\!2L \Ebb_{S} [f_S^*]}\log T \!+\! \sqrt{\!2Lf(\!\wb_{0}\!) \!+\!  2\Ebb_S[\nu_S^2]} \bigg) }. \nonumber
\end{align}
\end{thm}

%% file: composite.tex
\section{Regularized Nonconvex Optimization}\label{sec:regu}
In practical applications, regularization is usually applied to the risk minimization problem in order to either promote certain structures on the desired solution or to restrict the parameter space. In this section, we explore how regularization can improve the generation error, and hence help to avoid overfitting for SGD.

Here, for any weight $\lambda > 0$, we consider the regularized population risk minimization (R-PRM) and the regularized empirical risk minimization (R-ERM):
\begin{align}
	&\min_{\wb \in \Omega} ~ \Phi(\wb) := f(\wb) + \lambda h(\wb), \tag{R-PRM}\nonumber\\
	&\min_{\wb \in \Omega} ~ \Phi_S(\wb) := f_S(\wb) + \lambda h(\wb), \tag{R-ERM} \nonumber
\end{align}
where $h$ corresponds to the regularizer and $f, f_S$ are the population and empirical risks, respectively. In particular, we are interested in the following class of regularizers.
\begin{assum}\label{assum: h}
	The regularizer function $h$ is 1-strongly convex and nonnegative.
\end{assum}
Without loss of generality, we assume that the strongly convex parameter of $h$ is 1, and this can be adjusted by scaling the weight parameter $\lambda$.
Strongly convex regularizers are commonly used in machine learning applications, and typical examples include $\frac{\lambda}{2}\norm{\wb}^2$ for ridge regression, Tikhonov regularization $\frac{\lambda}{2}\norm{\Gamma\wb}^2$ and elastic net $\lambda_1\norm{\wb}_1 \!+\! \lambda_2\norm{\wb}^2$, etc. 
Here, we allow the regularizer $h$ to be non-differentiable (e.g., the elastic net), and introduce the following proximal mapping with parameter $\alpha >0$ to deal with the non-smoothness.
\begin{align}
	\prox{\alpha h}(\wb) := \arg\min_{\ub \in \Omega} h(\ub) + \frac{1}{2\alpha} \norm{\ub - \wb}^2.
\end{align}
The proximal mapping is the core of the proximal method for solving convex problems \cite{Parikh_2014,Beck_2009} and nonconvex ones \cite{Li_2017,Attouch_2013}. 
%%Moreover, it can be used to characterize the critical points of non-smooth objective functions. In specific, for a composite objective function $\Phi = f+ h$ where $f$ is differentiable and $h$ is convex and possibly non-smooth, $\wb$ is a critical point whenever the Fermat's rule is satisfied [Proposition 25.1 (iv), \cite{Bauschke_2011}], i.e.,
%\begin{align}
%	\zero \in \Gc_{\Phi}^{\alpha}(\wb):= \tfrac{1}{\alpha}\left(\wb -  \prox{\alpha h}(\wb - \alpha \nabla f (\wb))\right). 
%\end{align}
%Intuitively, $\Gc_{\Phi}(\wb)$ is a generalized gradient of the non-smooth function $\Phi$, and is referred to as the {\em gradient mapping}. 
%In the special case $h\equiv 0$, the gradient mapping reduces to the gradient $\nabla f$ of the differentiable function $f$. 
In particular, we apply the proximal SGD to solve the R-ERM. With the same notations as those defined in the previous section, the update rule of the proximal SGD can be written as, for $t = 0, \ldots, T-1$ 
\begin{align}
\wb_{t+1} = \prox{\alpha_t h}\big(\wb_{t} -  \alpha_t \nabla \ell (\wb_{t}; \zb_{\xi_t})\big). \tag{proximal-SGD}
\end{align}  
Similarly, we denote $\{\wb_{t, S}\}_t$ as the iterate sequence generated by the proximal SGD with the data set $S$. 

It is clear that the generalization error of the function value for the regularized risk minimization, i.e., $|\Phi(\wb_{T, S}) - \Phi_S(\wb_{T, S})|$, is the same as that for the un-regularized risk minimization. Hence, \Cref{thm: stability} is also applicable to the mean square generalization error of the regularized risk minimization.  
However, the development of the generalization error bound is different from the analysis in the previous section from two aspects. First, the analysis of the on-average iterate stability of the proximal SGD is technically more involved than that of SGD due to the possibly non-smooth regularizer. 
%We thus apply the technical tools developed in \cite{Ghadimi_2016} for analyzing the optimization path of the proximal SGD, and develop the bound in terms of the on-average variance of the stochastic gradients.
Secondly, the proximal mappings of strongly convex functions are strictly contractive (see item 2 of \Cref{prop: 1} in the appendix). Thus, the proximal step in the proximal SGD enhances the stability between the iterates $\wb_{t, S}$ and $\wb_{t, \overline{S}}$ that are generated by the algorithm using perturbed datasets, and this further improves the generalization error. 
The next result provides a quantitative statement.

 \begin{thm}\label{thm: improve_gene}
 Consider the regularized risk minimization. Suppose $\ell$ is nonconvex. Let Assumptions \ref{assum: loss}, \ref{assum: var} and \ref{assum: h} hold, and apply the proximal SGD to solve the R-ERM with the dataset $S$. Let $\lambda > L$ and $\alpha_t = \frac{c}{t+2}$ with $0<c<\frac{1}{L}$. Then, the following bound holds with probability at least $1- \delta$.
    \begin{align}
    	|\Phi(\wb_{T, S})  &- \Phi_S(\wb_{T, S})|  \nonumber\\
    	&\le  \sqrt{\frac{1}{n\delta} \Big( 2M^2  + \frac{24M\sigma}{(\lambda - L)} \sqrt{L\Phi(\wb_{0}) + \Ebb_{S}[\nu_S^2]} \Big)}. \nonumber
    \end{align}
 \end{thm}

\Cref{thm: improve_gene} provides {\em probabilistic} guarantee for the generalization error of the proximal SGD in terms of the on-average variance of the stochastic gradients. Comparison of \Cref{thm: improve_gene} with \Cref{coro: prob_gene} indicates that a strongly convex regularizer substantially improves the generalization error bound of SGD for nonconvex loss functions by removing the logarithm dependence on $T$. It is also interesting to compare \Cref{thm: improve_gene} with [Proposition 4 and Theorem 1, \cite{London_2017}], which characterize the generalization error of SGD for strongly convex functions with probabilistic guarantee. The two bounds have the same order in terms of $n$ and $T$, indicating that a strongly convex regularizer even improves the generalization error for a nonconvex function to be the same as that for a strongly convex function. 
In practice, the regularization weight $\lambda$ should be properly chosen to balance between the generalization error and the training loss, as otherwise the parameter space can be too restrictive to yield a good solution for the risk function. 

We further conduct experiments to explore the effect of regularization on the generalization error by adding the regularizer $\frac{\lambda}{2} \norm{\wb}^2$ to the objective functions. In particular, we apply the proximal SGD to solve the logistic regression (with dataset a9a) and train the neural network (with dataset MNIST) mentioned in the previous section. \Cref{fig: 2} shows the results where the left axis denotes the scale of the training error and the right axis denotes the scale of the generalization error. It can be seen that the corresponding generalization errors improve as the regularization weight gets large. This agrees with our theoretical finding on the impact of regularization. On the other hand, the training performances for both problems degrade as the regularization weight increases, which is reasonable because in such a case the optimization focuses too much on the regularizer and the obtained solution does not minimize the loss function well. Hence, there is a trade-off between the training and generalization performance in tuning the regularization parameter.
\begin{figure}[th]
	\centering 
	\subfigure{
		\includegraphics[width=2in]{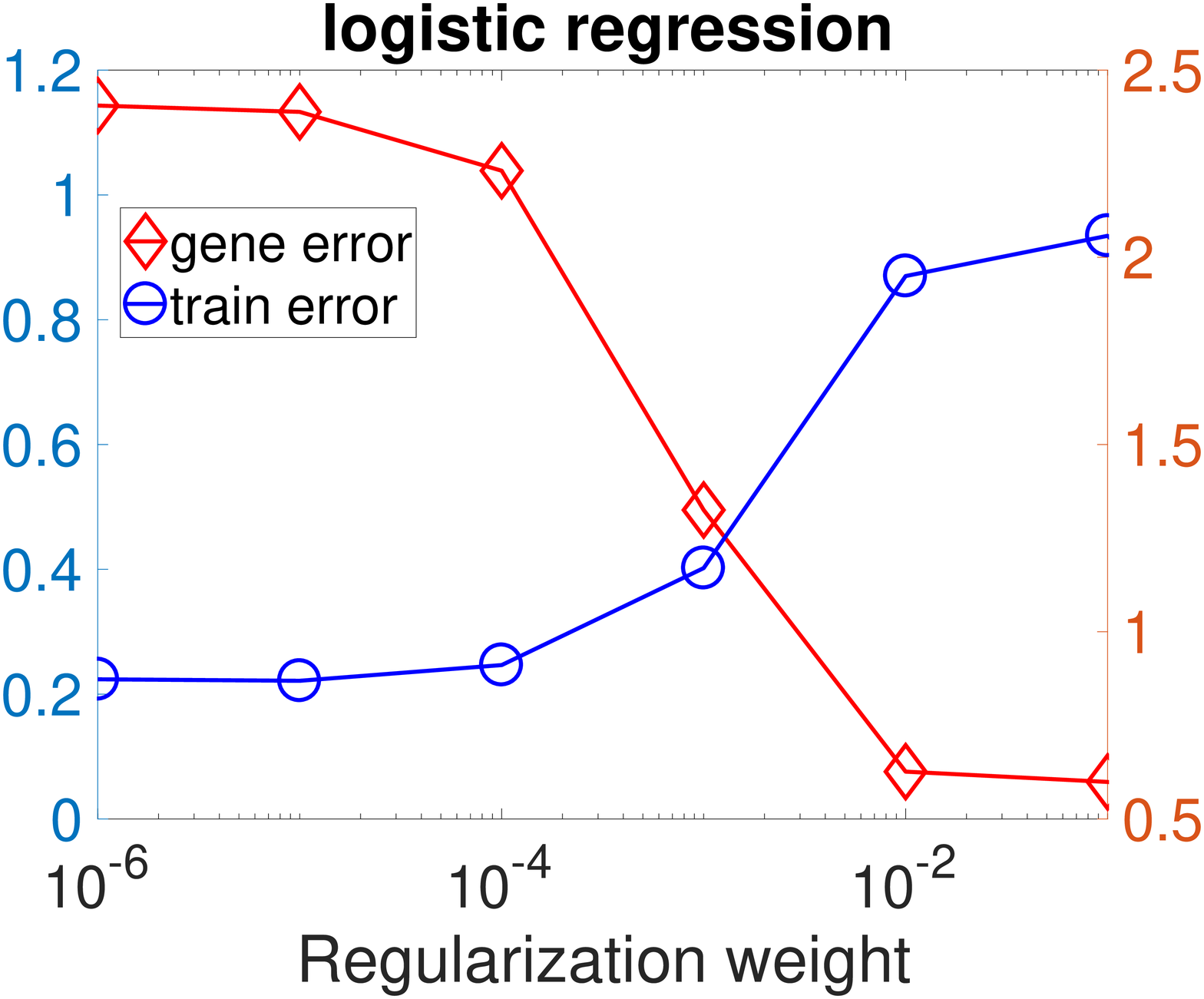}
	}
\vspace{0.5cm}
	\subfigure{
		\includegraphics[width=2in]{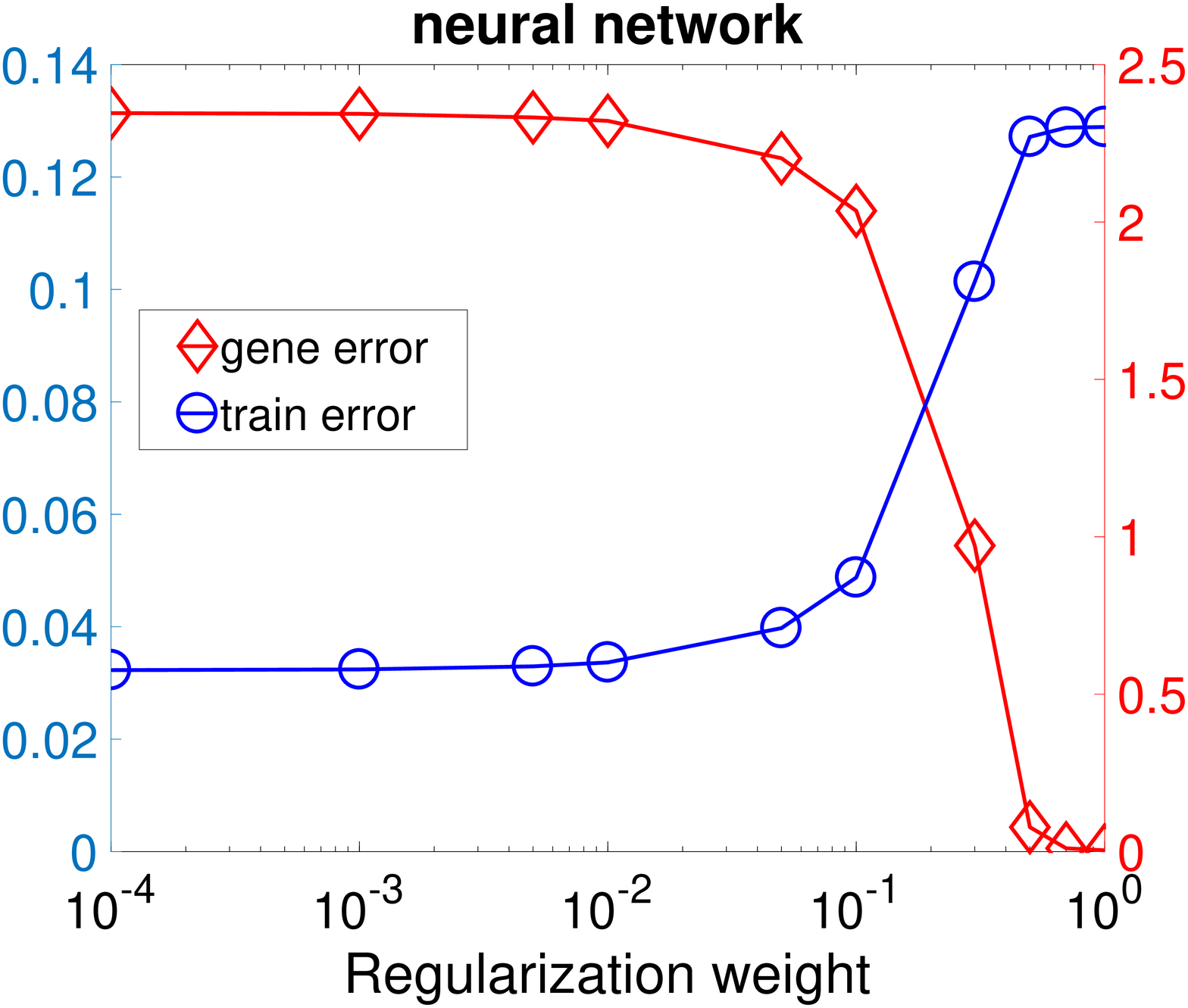}
	}
	\caption{Generalization error v.s. regularization parameter.} 
	\label{fig: 2}
	\vspace{-0.5cm}
\end{figure}

\subsection{Generalization Bound with High-Probability Guarantee}

The studies of the previous sections explore the {\em probabilistic} guarantee for the generalization errors of nonconvex loss functions and nonconvex loss functions with strongly convex regularizers. For example, apply SGD to solve a generic nonconvex loss function, then \Cref{coro: prob_gene} suggests that for any $\epsilon > 0$,
\begin{align*}
\mathbb{P} (|f(\wb_{T, S}) - f_S(\wb_{T, S})| > \epsilon ) < \Oc\Big( \frac{\log T}{n\epsilon^2}\Big),
\end{align*}
which decays sublinearly as $\frac{n}{\log T} \to \infty$. In this subsection, we study a stronger probabilistic guarantee for the generalization error, i.e., the probability for it to be less than $\epsilon$ decays {\em exponentially}. We refer to such a notion as high-probability guarantee. In particular, we explore for which cases of nonconvex loss functions we can establish such a stronger performance guarantee.  

Towards this end, we adopt the uniform stability framework proposed in \cite{Elisseeff_2005}. Note that \cite{Hardt_2016} also studied the uniform stability of SGD, but only characterized the generalization error in expectation, which is weaker than the {\em exponential} probabilistic concentrtion bound that we study here. 

Suppose we apply SGD with the same sample path $\bm{\xi}$ to solve the ERM with the datasets $S$ and $\overline{S}$, respectively, and
denote $\wb_{T, S, \bm{\xi}}$ and $\wb_{T, \overline{S}, \bm{\xi}}$ as the corresponding outputs. Also, suppose we apply the SGD with different sample paths $\bm{\xi}$ and $\overline{\bm{\xi}}$ to solve the same problem with the dataset $S$, respectively, and denote  $\wb_{T, S, \bm{\xi}}$ and $\wb_{T, S, \overline{\bm{\xi}}}$ as the corresponding outputs. Here, $\overline{\bm{\xi}}$ denotes the sample path that replaces one of the sampled indices, say $\xi_{t_0}$, with an i.i.d copy $\xi_{t_0}'$. The following result is a variant of [Theorem 15, \cite{Elisseeff_2005}].
\begin{lemma}\label{thm: eliss} 
	Let \Cref{assum: loss} hold. If SGD satisfies the following conditions \footnote{\Cref{thm: eliss} is slightly different from that in [Theorem 15, \cite{Elisseeff_2005}], in which $\overline{S}$ excludes a particular sample instead of replacing it. The proof follows the same idea and we omit it for simplicity.}
	\begin{align}
		&\sup_{S, \overline{S}, \zb} \Ebb_{\bm{\xi}} |\ell(\wb_{T, S, \bm{\xi}}; \zb) - \ell(\wb_{T, \overline{S}, \bm{\xi}}; \zb)| \le \beta, \nonumber\\
		&\sup_{\bm{\xi}, \overline{\bm{\xi}}, S, \zb} |\ell(\wb_{T, S, \bm{\xi}}; \zb) - \ell(\wb_{T, S, \overline{\bm{\xi}}}; \zb) | \le \rho. \nonumber
	\end{align}
	Then, the following bound holds with probability at least $1-\delta$.
	\begin{align}
		|\Phi(\wb_{T, S}) - &\Phi_S(\wb_{T, S})| \nonumber\\
		&\le  2\beta + \left(\frac{M+4n\beta}{\sqrt{2n}} + \sqrt{2T}\rho \right)\sqrt{\log\tfrac{2}{\delta}}. \nonumber
	\end{align}
\end{lemma}
 
%{\color{red} [Please check the following two equations]}
 Note that \Cref{thm: eliss} implies that 
 \begin{align*}
 \mathbb{P}(|\Phi(\wb_{T, S}) - \Phi_S(\wb_{T, S})|>\epsilon) \le \Oc\bigg(\exp\Big(\tfrac{-\epsilon^2}{\sqrt{n}\beta +\sqrt{T}\rho}\Big)\bigg).
 \end{align*}
 Hence, if $\beta =o(n^{-\frac{1}{2}})$ and $\rho =o(T^{-\frac{1}{2}})$, then we have exponential decay in probability as $n\rightarrow \infty$ and $T \rightarrow \infty$. 
It turns out that our analysis of the uniform stability of SGD for general nonconvex functions yields that 
$\beta =\Oc(n^{-1}),  \rho =\Oc(\log T)$, which does not lead to the desired high-probability guarantee for the generalization error. On the other hand, the analysis of the uniform stability of the proximal SGD for nonconvex loss functions with strongly convex regularizers yields that
$\beta = \Oc(n^{-1}), \rho =\Oc(T^{-c(\lambda - L)}),$
which leads to the high-probability guarantee if we choose $\lambda > L$ and $c> \frac{1}{2(\lambda - L)}$. This further demonstrates that a strongly convex regularizer can significantly improve the quality of the probabilistic bound for the generalization error. The following result is a formal statement of the above discussion.
%Thus, by analyzing the uniform stability for the functions of interest, and we conclude that only nonconvex function under a strongly convex regularizer has exponential decay in probability, whereas neither general nonconvex functions nor nonconvex functions that satisfy the gradient dominance condition has such a stronger guarantee. %Based on the above result, we obtain the following exponential high probability guarantee for the generalization error of the proximal SGD.
\begin{thm}\label{thm: high_1}
	Consider the regularized risk minimization with the nonconvex loss function $\ell$. Let Assumptions \ref{assum: loss} and \ref{assum: h} hold, and apply the proximal SGD to solve the R-ERM with the data set $S$. Choose $\lambda > L$ and $\alpha_t = \frac{c}{t+2}$ with $\frac{1}{2(\lambda - L)}<c<\frac{1}{\lambda - L}$. Then, the following bound holds with probability at least $1- \delta$
	\begin{align}
		|\Phi(\wb_{T, S}) &- \Phi_S(\wb_{T, S})| \nonumber\\
		&\le \Big(\frac{M}{\sqrt{n}} + \frac{4\sigma^2}{\sqrt{n}(\lambda - L)} + \frac{4\sigma^2 c}{T^{c(\lambda - L) - \frac{1}{2}}} \Big)\sqrt{\log\frac{2}{\delta}}. \nonumber
	\end{align}
\end{thm}
\Cref{thm: high_1} implies that 
\begin{align*}
	\mathbb{P}(|\Phi(\wb_{T, S}) &- \Phi_S(\wb_{T, S})|>\epsilon) \nonumber\\
	&\le \Oc\bigg(\exp\Big(\tfrac{-\epsilon^2}{n^{-\frac{1}{2}} + T^{\frac{1}{2}-c(\lambda - L)}}\Big)\bigg).
\end{align*}
 Hence, if we choose $c = \frac{1}{\lambda - L}$ and run the proximal SGD for $T = \Oc(n)$ iterations (i.e., constant passes over the data), then the probability of the event decays exponentially as $\Oc(\exp(-\sqrt{n}\epsilon^2))$.
 
The proof of \Cref{thm: high_1} characterizes the uniform iterate stability of the proximal SGD with regard to the perturbations of both the dataset and the sample path. Unlike the on-average stability in \Cref{coro: prob_gene} where the stochastic gradient norm is bounded by the on-average variance of the stochastic gradients, the uniform stability captures the worst case among all datasets, and hence uses the uniform upper bound $\sigma$ for the stochastic gradient norm. 

We note that [Theorem 3, \cite{London_2017}] also established a probabilistic bound under the PAC Bayesian framework. However, their result yields exponential concentration guarantee only for strongly convex loss functions. As a comparison, \Cref{thm: high_1} relaxes the requirement of strong convexity for loss functions to {\em nonconvex} loss functions with strongly convex regularizers, and hence serves as a complementary result to theirs. Also, \cite{Mou_2017} establishes the high-probability bound for the generalization error of SGD with regularization. However, their result holds only for the particular regularizer $\frac{1}{2}\norm{\wb}^2$, and high-probability bound holds only with regard to the random draw of the data. As a comparison, our result holds for all strongly convex regularizers, and the high-probability bound hold with regard to both the draw of data and randomness of algorithm.

%% file: supplementary.tex
\newpage
\appendix
\section{Proof of Main Results}\label{sec:append_A}
\subsection{Proof of \Cref{thm: stability}}
The proof is based on [Lemma 11, \cite{Elisseeff_2005}] and \Cref{assum: loss}. Denote $S^i$ as the data set that replaces the $i$-th sample of $S$ with an i.i.d.\ copy generated from the distribution $\mathcal{D}$. Following from Lemma 11 of \cite{Elisseeff_2005}, we obtain
\begin{align}
	\Ebb_{S, \bm{\xi}} |f_S (\wb_{T,S}) - f (\wb_{T,S})|^2 
	&\le \frac{2M^2}{n}  + \frac{12M}{n} \sum_{i=1}^{n} \Ebb_{\bm{\xi}, S, S^i} \left[|\ell(\wb_{T, S}; \zb_i) - \ell(\wb_{T, S^i}; \zb_i)|\right] \nonumber \\
	&\le \frac{2M^2}{n}  + \frac{12M\sigma}{n} \sum_{i=1}^{n} \Ebb_{\bm{\xi}, S, S^i} \norm{\wb_{T, S} - \wb_{T, S^i}} \nonumber \\
	&= \frac{2M^2}{n}  + 12M\sigma \Ebb_{\bm{\xi}, S, \overline{S}} \norm{\wb_{T, S} - \wb_{T, \overline{S}}}, \nonumber
\end{align}
where the second inequality uses the Lipschitz property of the loss function in \Cref{assum: loss}, and the last equality is due to the fact that the perturbed samples in $S^i$ and $\overline{S}$ are generated i.i.d from the underlying distribution.

\subsection{Proof of \Cref{coro: prob_gene}}
The proof is based on the following two important lemmas, which we prove first.

%% Lemma 2
\begin{lemma}\label{lemma: GE_general}
	Let \Cref{assum: loss} hold. Apply SGD with the same sample path $\bm{\xi}$ to solve the ERM with data sets $S$ and $\overline{S}$, respectively. Choose $\alpha_t = \frac{c}{(t+2)\log (t+2)}$ with $0<c<\frac{1}{L}$, then the following bound holds.
	\begin{align}
	\Ebb_{S, \overline{S}, \bm{\xi}} [\delta_{t+1, S, \overline{S}}] 
	&\le (1 + \alpha_t L) \Ebb_{S, \overline{S}, \bm{\xi}} [\delta_{t, S, \overline{S}}]  + \frac{2\alpha_t}{n}\Ebb_{S, \bm{\xi}} \left[\norml{\nabla \ell (\wb_{t, S}; \zb_{1})}\right]. \nonumber
	\end{align} 
\end{lemma}
\begin{proof}[Proof of \Cref{lemma: GE_general}]
	Consider the two fixed data sets $S$ and $\overline{S}$ that differ at, say, the first data sample. At the $t$-th iteration, we consider two cases of the sampled index ${\xi_t}$. In the first case, $1 \notin {\xi_t}$ (w.p. $\frac{n-1}{n}$), i.e., the sampled data from $S$ and $\overline{S}$ are the same, and we obtain that  
	\begin{align}
	\delta_{t+1, S, \overline{S}} 
	&= \norml{\wb_{t, S} - \alpha_t \nabla \ell (\wb_{t, S}; \zb_{\xi_t}) - \wb_{t,\overline{S}} + \alpha_t \nabla \ell (\wb_{t, \overline{S}}; \zb_{\xi_t})} \nonumber\\
	&\le \delta_{t, S, \overline{S}} + \alpha_t \norml{\nabla \ell (\wb_{t, S}; \zb_{\xi_t}) - \nabla \ell (\wb_{t, \overline{S}}; \zb_{\xi_t}) } \nonumber\\
	&\le (1 + \alpha_t L)\delta_{t, S, \overline{S}},  \label{eq: 3}
	\end{align}
	where the last inequality uses the $L$-Lipschitz property of $\nabla \ell$. In the other case, $1 \in {\xi_t}$ (w.p. $\frac{1}{n}$), we obtain that
	\begin{align}
	\delta_{t+1, S, \overline{S}} 
	&= \norml{\wb_{t, S} - \alpha_t  \nabla \ell (\wb_{t, S}; \zb_{1}) - \wb_{t,\overline{S}} + \alpha_t \nabla \ell (\wb_{t, \overline{S}}; \zb_{1}')} \nonumber\\
	&\le \delta_{t, S, \overline{S}} + \alpha_t \norml{\nabla \ell (\wb_{t, S}; \zb_{1}) - \nabla \ell (\wb_{t, \overline{S}}; \zb_{1}')}\nonumber\\
	&\le \delta_{t, S, \overline{S}} + \alpha_t \left(\norml{\nabla \ell (\wb_{t, S}; \zb_{1})} + \norm{\nabla \ell (\wb_{t, \overline{S}}; \zb_{1}')} \right). \label{eqq: 4}
	\end{align}
	Combining the above two cases and taking expectation with respect to all randomness, we obtain that
	\begin{align}
	\Ebb_{S, \overline{S}, \bm{\xi}} [\delta_{t+1, S, \overline{S}}] 
	&\le \left[\frac{n-1}{n} (1 + \alpha_t L) + \frac{1}{n}   \right]  \Ebb_{S, \overline{S}, \bm{\xi}} [\delta_{t, S, \overline{S}}] + \frac{1}{n} \alpha_t\Ebb_{S, \overline{S}, \bm{\xi}}\left(\norml{\nabla \ell (\wb_{t, S}; \zb_{1})} + \norm{\nabla \ell (\wb_{t, \overline{S}}; \zb_{1}')} \right) \nonumber\\
	&\overset{(i)}{\le} (1 + \alpha_t L) \Ebb_{S, \overline{S}, \bm{\xi}} [\delta_{t, S, \overline{S}}]  + \frac{2\alpha_t}{n}\Ebb_{S, \bm{\xi}} \left[\norml{\nabla \ell (\wb_{t, S}; \zb_{1})}\right], \label{eq: 38} 
	\end{align}
	where (i) uses the fact that $\zb_{1}'$ is an i.i.d.\ copy of $\zb_{1}$.
\end{proof}

%% Lemma 3
\begin{lemma}\label{lemma: grad_bound}
	Let Assumptions \ref{assum: loss} and \ref{assum: var} hold. Apply SGD to solve the ERM with data set $S$ and choosing $\alpha_t \le \frac{c}{t+2}$ for some $0<c<\frac{1}{L}$. Then, the following bound holds.
	\begin{align}
	&\Ebb_{\bm{\xi}, S} \left[\norm{\nabla \ell (\wb_{t, S}; \zb_{1})}\right] \le  \sqrt{2Lf(\wb_{0}) + \frac{1}{2}\Ebb_{S}[\nu_S^2] }. \nonumber
	\end{align}
\end{lemma}
\begin{proof}[Proof of \Cref{lemma: grad_bound}]
	By \Cref{assum: loss}, $\ell$ is nonnegative and $\nabla\ell$ is $L$-Lipschitz. Then, eq. (12.6) of \cite{Shalev-Shwartz_2014} shows that
	\begin{align} 
	\forall \wb, \quad \norm{\nabla \ell (\wb; \zb)} \le \sqrt{2L \ell(\wb; \zb)}. \label{eq: 33}
	\end{align}
	Based on \cref{eq: 33}, we further obtain that
	\begin{align}
	\Ebb_{\bm{\xi}, S} \norm{\nabla \ell (\wb_{t, S}; \zb_{1})} 
	&\le \sqrt{2L} \Ebb_{\bm{\xi}, S} \sqrt{\ell(\wb_{t, S}; \zb_{1})} \overset{(i)}{\le} \sqrt{2L}  \sqrt{\Ebb_{\bm{\xi}, S}\ell(\wb_{t, S}; \zb_{1})} \nonumber\\
	&\overset{(ii)}{\le} \sqrt{2L}  \sqrt{\Ebb_{\bm{\xi}, S} \frac{1}{n}\sum_{j=1}^{n}\ell(\wb_{t, S}; \zb_{j})} = \sqrt{2L}  \sqrt{\Ebb_{\bm{\xi}, S} f_S(\wb_{t, S})}, \label{eq: 34}
	\end{align}
	where (i) uses the Jesen's inequality and (ii) uses the fact that all samples in $S$ are generated i.i.d.\ from $\mathcal{D}$. 
	
	Next, consider a fixed data set $S$ and denote $\gb_{t, S} = \nabla \ell (\wb_{t, S}; \zb_{\xi_t})$ as the sampled stochastic gradient at iteration $t$. Then, by smoothness of $\ell$ and the update rule of the SGD, we obtain that
	\begin{align}
	f_S(\wb_{t+1, S}) - f_S(\wb_{t, S}) &\le \inner{\wb_{t+1, S} - \wb_{t, S}}{\nabla f_S(\wb_{t, S})} + \frac{L}{2} \norm{\wb_{t+1, S} - \wb_{t, S}}^2 \nonumber\\
	&= \inner{-\alpha_t \gb_{t, S} }{\nabla f_S(\wb_{t, S})} + \frac{L\alpha_t^2}{2} \norml{\gb_{t, S}}^2. \nonumber 
	\end{align}
	Conditioning on $\wb_{t, S}$ and taking expectation with respect to $\bm{\xi}$, we further obtain from the above inequality that
	\begin{align}
	\Ebb_{\bm{\xi}} &\left[f_S(\wb_{t+1, S}) - f_S(\wb_{t, S}) |\wb_{t, S}\right] \nonumber \\
	& \le \left(\frac{L\alpha_t^2}{2} - \alpha_t \right) \norml{\nabla f_S(\wb_{t, S})}^2 + \frac{L\alpha_t^2}{2}  \Ebb_{\bm{\xi}}\left[\norml{\gb_{t, S}}^2 - \norml{\nabla f_S(\wb_{t, S})}^2|\wb_{t, S}\right]. \label{eq: 37}
	\end{align}
Note that $\frac{L\alpha_t^2}{2} - \alpha_t < 0$ by our choice of stepsize. Further taking expectation with respect to the randomness of $\wb_{t, S}$ and $S$, and telescoping the above inequality over $0, \ldots, t-1$, we obtain that 
	\begin{align}
	\Ebb_{\bm{\xi}, S} \left[f_S(\wb_{t, S})\right] 
	&\overset{(i)}{\le} \Ebb_{S}f_S(\wb_{0}) + \sum_{t'=0}^{t-1} \frac{L\alpha_{t'}^2}{2} \Ebb_{S}[\nu_S^2] \nonumber\\
	&= f(\wb_{0}) + \sum_{t'=0}^{t-1} \frac{Lc^2\Ebb_{S}[\nu_S^2]}{2(t'+2)^2}  \overset{(ii)}{\le} f(\wb_{0}) + \frac{Lc^2\Ebb_{S}[\nu_S^2]}{4}, \nonumber
	\end{align}
	where (i) uses the fact that the variance of the stochastic gradients is bounded by $\Ebb_{S}[\nu_S^2]$, and (ii) upper bounds the summation by the integral, i.e., $ \sum_{t'=0}^{t-1} \frac{1}{(t'+2)^2} \lesssim \int_{1}^{t} \frac{1}{t'^2} dt'$.
	Substituting the above result into \cref{eq: 34} and noting that $cL \le 1$, we obtain the desired result.
\end{proof}

%% Main proof

Now by \Cref{lemma: GE_general}, we obtain that
\begin{align}
\Ebb_{S, \overline{S}, \bm{\xi}} [\delta_{t+1, S, \overline{S}}] 
&\le (1 + \alpha_t L) \Ebb_{S, \overline{S}, \bm{\xi}} [\delta_{t, S, \overline{S}}] + \frac{2\alpha_t}{n}\Ebb_{S, \bm{\xi}} \left[\norml{\nabla \ell (\wb_{t, S}; \zb_{1})}\right] \nonumber\\
&\overset{(i)}{\le} (1 + \alpha_t L) \Ebb_{S, \overline{S}, \bm{\xi}} [\delta_{t, S, \overline{S}}]  + \frac{2\alpha_t  \sqrt{2Lf(\wb_{0}) + \frac{\Ebb_{S}[\nu_S^2]}{2}}}{n}, \label{eq: 5}
\end{align}
where (i) applies \Cref{lemma: grad_bound}.
Recursively applying \cref{eq: 5} over $t = 0, \ldots, T-1$ and noting that $\delta_{0} = 0$ and $\alpha_t = \frac{c}{(t+2)\log(t+2)}$, we obtain
\begin{align}
\Ebb_{S, \overline{S}, \bm{\xi}} [\delta_T] 
&\le  \sum_{t=0}^{T-1} \left[\prod_{k=t+1}^{T-1} (1+\alpha_k L)  \right] \frac{2c \sqrt{2Lf(\wb_{0}) + \frac{\Ebb_{S}[\nu_S^2]}{2}}}{(t+2){\log (t+2)}n} \nonumber\\
&\overset{(i)}{\le} \sum_{t=0}^{T-1} \left[\exp\left( \sum_{k=t+1}^{T-1} \frac{cL}{(k+2) \log (k+2)}\right) \right] \frac{2c \sqrt{2Lf(\wb_{0}) + \frac{\Ebb_{S}[\nu_S^2]}{2}}}{(t+2){\log (t+2)}n} \nonumber\\
&\overset{(ii)}{\le} \sum_{t=0}^{T-1} \left( \frac{\log T}{\log (t+2)} \right)^{cL}  \frac{2c \sqrt{2Lf(\wb_{0}) + \frac{\Ebb_{S}[\nu_S^2]}{2}}}{(t+2){\log (t+2)}n} \nonumber\\
&\overset{(iii)}{\le} \frac{2c \sqrt{2Lf(\wb_{0}) + \frac{\Ebb_{S}[\nu_S^2]}{2}}}{n} \log T, \nonumber
\end{align}
where (i) uses the fact that $1+x \le \exp(x)$. For (ii) and (iii), we apply the integral upper bounds to bound the summations, i.e., $\sum_{k=t+1}^{T-1} \frac{cL}{(k+2) \log (k+2)} \lesssim \int_{t}^{T} \frac{cL}{k\log k} d k, \sum_{t=0}^{T-1} (t+2)^{-1} \log^{-1-cL}(t+2) \lesssim \int_{t=1}^{T} t^{-1} \log^{-1-cL}t d t$, and use the fact that $cL<1$. Substituting the above result into \Cref{thm: stability} and applying the Chebyshev's inequality yields the desired result.

%-----------------------------------------------------------
\subsection*{Proof of \Cref{thm: gene_gd}}

We first prove a useful lemma.
%%% lemma
\begin{lemma}\label{lemma: grad_bound_GD}
	Let Assumptions \ref{assum: loss} and \ref{assum: var} hold. Apply the SGD to solve the ERM with data set $S$, where $f_S$ is $\gamma$-gradient dominant ($\gamma<L$) with the minimum function value $f_S^*$. Suppose we choose $\alpha_t \le \frac{c}{t+2}$ for some $0<c<\min\{\frac{2}{\gamma}, \frac{1}{L}\}$. Then the following bound holds.
\begin{align}
	&\Ebb_{\bm{\xi}, S} \left[\norm{\nabla \ell (\wb_{t, S}; \zb_{1})}\right] \le \sqrt{2L \Ebb_{S} [f_S^*] + \frac{1}{t} \left(2Lf(\wb_{0}) + 2\Ebb_S[\nu_S^2]  \right)}. \nonumber
	\end{align}
\end{lemma}
\begin{proof}[Proof of \Cref{lemma: grad_bound_GD}]
We first note that \cref{eq: 34} and \cref{eq: 37} both hold here, which we rewritten below for convenience.
\begin{align}
	\Ebb_{\bm{\xi}, S} \norm{\nabla \ell (\wb_{t, S}; \zb_{1})} &\le \sqrt{2L}  \sqrt{\Ebb_{\bm{\xi}, S} f_S(\wb_{t, S})}, \label{eq: 34_GD}
	\end{align}
\begin{align}
	\Ebb_{\bm{\xi}} &\left[f_S(\wb_{t+1, S}) - f_S(\wb_{t, S}) |\wb_{t, S}\right] \nonumber \\
	& \le \left(\frac{L\alpha_t^2}{2} - \alpha_t \right) \norml{\nabla f_S(\wb_{t, S})}^2 + \frac{L\alpha_t^2}{2}  \Ebb_{\bm{\xi}}\left[\norml{\gb_{t, S}}^2 - \norml{\nabla f_S(\wb_{t, S})}^2|\wb_{t, S}\right]. \label{eq: 37_GD}
	\end{align}
	
Following from \cref{eq: 37_GD} and the fact that $f_S$ is $\gamma$-gradient dominant, we obtain	
\begin{align}
	&\Ebb_{\bm{\xi}} \left[f_S(\wb_{t+1, S}) - f_S(\wb_{t, S}) ~|~ \wb_{t, S}\right] \nonumber\\
	&\le \left(\frac{L\alpha_t^2}{2} - \alpha_t \right)2\gamma (f_S(\wb_{t, S}) - f_S^*)  + \frac{L\alpha_t^2}{2}  \Ebb_{\bm{\xi}}\left[\norml{g_{t, S}}^2 - \norml{\nabla f_S(\wb_{t, S})}^2~|~\wb_{t, S}\right].
	\end{align}
	Further taking expectation with respect to the randomness of $\wb_{t, S}$ and $S$, we obtain from the above inequality that
	\begin{align}
	\Ebb_{\bm{\xi}, S} \left[f_S(\wb_{t+1, S}) - f_S(\wb_{t, S}) \right] &\le \left(\frac{L\alpha_t^2}{2} - \alpha_t \right)2\gamma \Ebb_{\bm{\xi}, S}(f_S(\wb_{t, S}) - f_S^*)  + \frac{L\alpha_t^2}{2}  \Ebb_{S}[\nu_S^2] \nonumber\\
	&\le -\alpha_{t}\gamma\Ebb_{\bm{\xi}, S}(f_S(\wb_{t, S}) - f_S^*)+ \frac{L\alpha_t^2\Ebb_{S}[\nu_S^2]}{2}, \nonumber
	\end{align}
	where the last inequality uses the fact that $\frac{L\alpha_t^2}{2} \le \alpha_{t} /2$ for $c<\frac{1}{L}$.
	Rearranging the above inequality, we further obtain that
	\begin{align}
	\Ebb_{\bm{\xi}, S} \left[f_S(\wb_{t+1, S}) - f_S^* \right] 
	&\le (1-\alpha_{t} \gamma) \Ebb_{\bm{\xi}, S}(f_S(\wb_{t, S}) - f_S^*) + \frac{L\alpha_t^2\nu^2}{2} \nonumber\\
	&\le \prod_{t'=0}^{t}(1-\alpha_{t'}\gamma)\Ebb_{S}(f_S(\wb_{0}) - f_S^*) + \sum_{t'=0}^{t} \prod_{k=t' + 1}^{t-1}(1-\alpha_k\gamma) \frac{L\alpha_{t'}^2\Ebb_{S}[\nu_S^2]}{2}  \nonumber\\
	&\overset{(i)}{\le} t^{-c\gamma} \Ebb_{S}(f_S(\wb_{0}) - f_S^*) + \frac{Lc^2\Ebb_{S}[\nu_S^2]}{t^{c\gamma}} \nonumber\\
	&\overset{(ii)}{\le} \frac{1}{t^{c\gamma}} \left[f(\wb_{0}) +  Lc^2\Ebb_{S}[\nu_S^2]\right], \nonumber
	\end{align}
	where (i) uses the fact that $1-x\le \exp(-x)$ and upper bounds the summations by the corresponding integrals, i.e., $\exp(-c\gamma\sum_{t'=0}^{t} \frac{1}{t'+2}) \lesssim \exp(-c\gamma \int_{0}^{t} \frac{1}{t'}dt') $ and (ii) uses the fact that $c\gamma < 1/2$.
	We then conclude that 
	$$\Ebb_{\bm{\xi}, S} f_S(\wb_{t, S}) \le \Ebb_{S} [f_S^*] + \frac{1}{t^{c\gamma}} \left[f(\wb_{0}) +  L\nu^2c^2\right].$$ Substituting this bound into \cref{eq: 34_GD} and noting that $cL\le 1$, we obtain the desired result.
\end{proof}
%%%

To continue our proof, by \Cref{lemma: GE_general},
% in \Cref{sec: support}, 
we obtain that
\begin{align}
\Ebb_{S, \overline{S}, \bm{\xi}} [\delta_{t+1, S, \overline{S}}] 
&\le (1 + \alpha_t L) \Ebb_{S, \overline{S}, \bm{\xi}} [\delta_{t, S, \overline{S}}] + \frac{2\alpha_t}{n}\Ebb_{S, \bm{\xi}} \left[\norml{\nabla \ell (\wb_{t, S}; \zb_{1})}\right] \nonumber\\
&\le  (1 + \alpha_t L) \Ebb_{S, \overline{S}, \bm{\xi}} [\delta_{t, S, \overline{S}}] +\frac{2\alpha_t }{n}\sqrt{2L \Ebb_{S} [f_S^*] + \frac{1}{t^{c\gamma}} \left(2Lf(\wb_{0}) +  2\Ebb_S[\nu_S^2]  \right)},  \label{eq: 9}
\end{align}
where the last line applies \Cref{lemma: grad_bound_GD}.
Applying \cref{eq: 9} recursively over $t = 0, \ldots, T-1$ and noting that $\delta_{0} = 0, \alpha_t = \frac{c}{(t+2)\log(t+2)}$, we obtain that
\begin{align}
\Ebb_{S, \overline{S}, \bm{\xi}} [\delta_T ] 
&\le \sum_{t=0}^{T-1} \left[\prod_{k=t+1}^{T-1} (1+\alpha_k L)  \right] \frac{2c }{(t+2)\log(t+2)n}\sqrt{2L \Ebb_{S} [f_S^*] + \frac{1}{t^{c\gamma}} \left(2Lf(\wb_{0}) +  2\Ebb_S[\nu_S^2]  \right)} \nonumber\\
&\le \frac{2c}{n}\sum_{t=0}^{T-1} \left( \frac{\log T}{\log (t+2)} \right)^{cL} \frac{\sqrt{2L \Ebb_{S} [f_S^*]} + \sqrt{\frac{1}{t^{c\gamma}} \left(2Lf(\wb_{0}) +  2\Ebb_S[\nu_S^2] \right)}}{(t+2)\log(t+2)} \nonumber\\
&\le \frac{2c}{n} \left( \sqrt{2L \Ebb_{S} [f_S^*]}\log T + \sqrt{2Lf(\wb_{0}) +  2\Ebb_S[\nu_S^2]}\right). \nonumber
\end{align}

Substituting the above result into \Cref{thm: stability} yields the desired result.

%---------------------------
\subsection{Proof of \Cref{thm: improve_gene}}
Consider the fixed data sets $S$ and $\overline{S}$ that are differ at the first sample. At the $t$-th iteration, if $1 \notin {\xi_t}$ (w.p. $\frac{n-1}{n}$), we obtain that  
\begin{align}
\delta_{t+1, S, \overline{S}} 
&= \norml{\prox{\alpha_t h}\left(\wb_{t, S} - \alpha_t \nabla \ell (\wb_{t, S}; \zb_{\xi_t})\right) - \prox{\alpha_t h}\left(\wb_{t,\overline{S}} - \alpha_t \nabla \ell (\wb_{t, \overline{S}}; \zb_{\xi_t})\right)} \nonumber\\
&\overset{(i)}{\le} \frac{1}{1 + \alpha_t \lambda} \norml{\wb_{t, S} - \alpha_t \nabla \ell (\wb_{t, S}; \zb_{\xi_t}) - \wb_{t,\overline{S}} + \alpha_t  \nabla \ell (\wb_{t, \overline{S}}; \zb_{\xi_t})} \nonumber\\
&\le  \frac{1 + \alpha_t L}{1 + \alpha_t \lambda} \delta_{t, S, \overline{S}}, \label{eq: 29}
\end{align}
where (i) uses item 2 of \Cref{prop: 1}. On the other hand, if $1 \in {\xi_t}$ (w.p. $\frac{1}{n}$), we obtain that
\begin{align}
\delta_{t+1, S, \overline{S}} 
&= \bigg\| \prox{\alpha_t h}\left(\wb_{t, S} - \alpha_t  \nabla \ell (\wb_{t, S}; \zb_{1})\right) - \prox{\alpha_t h} \left(\wb_{t,\overline{S}} - \alpha_t \nabla \ell (\wb_{t, \overline{S}}; \zb_{1}') \right) \bigg\| \nonumber\\
&\overset{(i)}{\le} \frac{1}{1+\alpha_t \lambda} \norml{\wb_{t, S} - \alpha_t  \nabla \ell (\wb_{t, S}; \zb_{1}) - \wb_{t,\overline{S}} + \alpha_t \nabla \ell (\wb_{t, \overline{S}}; \zb_{1}')} \nonumber\\
&\le \frac{1}{1+\alpha_t \lambda} \delta_{t, S, \overline{S}} + \frac{\alpha_t}{1+\alpha_t \lambda}\left(\norm{\nabla \ell (\wb_{t, S}; \zb_{1})} + \norm{\nabla \ell (\wb_{t, \overline{S}}; \zb_{1}')} \right),  \label{eq: 28}
%&\le \frac{1 + \alpha_t L}{1+\alpha_t \lambda} \delta_{t, S, \overline{S}} + \frac{\alpha_t}{m}\frac{1}{1+\alpha_t \lambda} \left(\norm{\nabla \ell (\wb_{t, S}; \zb_{1})} + \norm{\nabla \ell (\wb_{t, \overline{S}}; \zb_{1}')} \right),
\end{align}
where (i) uses item 2 of \Cref{prop: 1}. Combining the above two cases and taking expectation with respect to the randomness of $\bm{\xi}$, $S$ and $\overline{S}$, we obtain that
\begin{align}
	\Ebb_{S, \overline{S}, \bm{\xi}} [\delta_{t+1, S, \overline{S}}] 
	&\le \left[\frac{n-1}{n}\frac{1 + \alpha_t L}{1+\alpha_t \lambda} + \frac{1}{n}\frac{1}{1+\alpha_t \lambda}  \right] 	\Ebb_{S, \overline{S}, \bm{\xi}} [\delta_{t, S, \overline{S}}] + \frac{1}{n}  \frac{2\alpha_t}{1+\alpha_t \lambda}\Ebb_{S, \bm{\xi}} \norm{\nabla \ell (\wb_{t, S}; \zb_{1})} \nonumber\\
	&\overset{(i)}{\le} \frac{1 + \alpha_t L}{1+\alpha_t \lambda} \Ebb_{S, \overline{S}, \bm{\xi}} [\delta_{S, \overline{S}, \bm{\xi}}] +  \frac{2\alpha_t}{n}\frac{1}{1+\alpha_t \lambda} \sqrt{2L\Phi(\wb_{0}) + 2\Ebb_{S}[\nu_S^2] \log t} \nonumber\\
	&\lesssim \exp(\alpha_t (L-\lambda)) \Ebb_{S, \overline{S}, \bm{\xi}} [\delta_{S, \overline{S}, \bm{\xi}}] + \frac{2\alpha_t}{n} \sqrt{2L\Phi(\wb_{0}) + 2\Ebb_{S}[\nu_S^2] \log t}, \nonumber
\end{align}
where (i) uses \Cref{lemma: grad_bound2}.
Recursively applying the above inequality over $t = 0, \ldots, T-1$ and noting that $\delta_{0} = 0, \alpha_{t} = \frac{c}{t+2}$, we obtain that
\begin{align}
\Ebb_{S, \overline{S}, \bm{\xi}} [\delta_{T, S, \overline{S}}] 
&\le  \sum_{t=0}^{T-1} \left[\prod_{k=t+1}^{T-1} \exp(\alpha_k (L-\lambda))  \right] \frac{2c \sqrt{2L\Phi(\wb_{0}) + 2\Ebb_{S}[\nu_S^2] \log t}}{(t+2)n} \nonumber\\
&\overset{(i)}{\le} \sum_{t=0}^{T-1} \left(\frac{t+2}{T}\right)^{c(\lambda - L)} \frac{2c \sqrt{2L\Phi(\wb_{0}) + 2\Ebb_{S}[\nu_S^2]}}{(t+2)n}  \log t\nonumber\\
&\overset{(ii)}{\le} \frac{2}{n(\lambda - L)} \sqrt{2L\Phi(\wb_{0}) + 2\Ebb_{S}[\nu_S^2]},\nonumber
\end{align}
where the $\log t$ term in (i) is ignored as it is order-wise smaller than other polynomial terms (In particular, for any $\delta > 0$ we have $\lim_{t\to \infty} \log t / t^{\delta} = 0$), and (ii) further upper bounds the summation with the integral, i.e., $\sum_{t=0}^{T-1} (t+2)^{c(\lambda - L) - 1} \lesssim \int_{1}^{T} t^{c(\lambda - L) - 1} dt$, and uses the fact that $c<\frac{1}{L}$. Then, applying \Cref{thm: stability} to the regularized risk minimization, we further obtain that
\begin{align}
	&\Ebb_{\bm{\xi}, S} \left[|\Phi_S (\wb_{T,S}) - \Phi (\wb_{T,S})|^2\right] 
	\le \frac{1}{n} \left( 2M^2  + \frac{24M\sigma}{(\lambda - L)} \sqrt{L\Phi(\wb_{0}) + \Ebb_{S}[\nu_S^2]} \right). \nonumber 
%	&\Ebb_{\bm{\xi}, S} \left[\norml{\Gc_{\Phi}^{\alpha_T}(\wb_{T, S}) - \Gc_{\Phi_S}^{\alpha_T}(\wb_{T, S})}^2 \right]
%	\le \Oc \left( \frac{1}{n} \left(\sigma^2 + \sigma \sqrt{\frac{\Ebb_{S}[\nu_S^2]}{m}}\right) \right). \nonumber
\end{align}
The desired result then follows by applying Chebyshev's inequality.

\subsection{Proof of \Cref{thm: high_1}}
The idea of the proof is to apply \Cref{thm: eliss} by developing the uniform stability bounds $\beta$ and $\gamma$. The proof also applies two useful lemmas on the proximal SGD.

We first evaluate $\beta$. Following the proof logic of \Cref{thm: improve_gene} and replacing the bound for the on-average stochastic gradient norm $\Ebb_{S, \bm{\xi}} \norm{\nabla \ell (\wb_{t, S}; \zb_{1})}$ with the uniform upper bound $\sigma$, we obtain that
\begin{align}
\sup_{S, \overline{S}, \zb} \Ebb_{\bm{\xi}} |\ell(\wb_{T, S}; \zb) - \ell(\wb_{T, \overline{S}}; \zb)| \le \sigma  \sup_{S, \overline{S}, \zb} \Ebb_{\bm{\xi}} [\delta_{T, S, \overline{S}}] \le \frac{2\sigma^2}{n(\lambda - L)}  := \beta. \nonumber
\end{align}
Next, we evaluate $\rho$. Consider any two sample paths $\bm{\xi}:=\{{\xi_1}, \ldots, {\xi_{t_0}}, \ldots, {\xi_{T-1}} \}$ and $\overline{\bm{\xi}}:= \{{\xi_1}, \ldots, {\xi_{t_0}'}, \ldots, {\xi_{T-1}}\}$, which are different at the $t_0$-th mini-batch. Note that
\begin{align}
 \sup_{\bm{\xi}, \overline{\bm{\xi}}, S, \zb} |\ell(\wb_{T, S, \bm{\xi}}; \zb) - \ell(\wb_{T, S, \overline{\bm{\xi}}}; \zb) | \le \sup_{\bm{\xi}, \overline{\bm{\xi}}, S, \zb}\sigma \norm{\wb_{T, S, \bm{\xi}} - \wb_{T, S, \overline{\bm{\xi}}}}.  \label{eq: 30}
\end{align}
Since the two sample paths only differ at the $t_0$-th iteration, we have that $\wb_{t, S, \bm{\xi}} - \wb_{t, S, \overline{\bm{\xi}}} = \zero$ for $t = 0, \ldots, t_0$. 
In particular, for $t = t_0$ we obtain that 
\begin{align}
&\norm{\wb_{t_0+1, S, \bm{\xi}} - \wb_{t_0+1, S, \overline{\bm{\xi}}}} \nonumber\\
&= \norml{\prox{\alpha_{t_0} h}\left(\wb_{t_0, S, \bm{\xi}} - \alpha_{t_0} \nabla \ell (\wb_{t_0, S, \bm{\xi}}; \zb_{\xi_{t_0}})\right) - \prox{\alpha_{t_0} h}\left(\wb_{t_0, S, \overline{\bm{\xi}}} - \alpha_{t_0} \nabla \ell (\wb_{t_0, S, \overline{\bm{\xi}}}; \zb_{\xi'_{t_0}})\right)} \nonumber\\
&\overset{(i)}{\le} \frac{1}{1 + \alpha_{t_0} \lambda} \norml{\wb_{t_0, S, \bm{\xi}} - \alpha_{t_0} \nabla \ell (\wb_{t_0, S, \bm{\xi}}; \zb_{\xi_{t_0}})- \wb_{t_0, S, \overline{\bm{\xi}}} + \alpha_{t_0} \nabla \ell (\wb_{t_0, S, \overline{\bm{\xi}}}; \zb_{\xi'_{t_0}})} \nonumber\\
&= \frac{1}{1 + \alpha_{t_0} \lambda} \norml{\alpha_{t_0} \nabla \ell (\wb_{t_0, S, \bm{\xi}}; \zb_{\xi_{t_0}})- \alpha_{t_0} \nabla \ell (\wb_{t_0, S, \overline{\bm{\xi}}}; \zb_{\xi'_{t_0}})} \nonumber\\
&\overset{(ii)}{\le} 2\alpha_{t_0}\sigma, \nonumber
\end{align}
where (i) uses \Cref{prop: 1} and (ii) uses the $\sigma$-bounded property of $\norm{\nabla \ell}$.
Now consider $t > t_0 + 1$. Note that in this case the sampled indices in $\bm{\xi}$ and $\overline{\bm{\xi}}$ are the same, and we further obtain that 
\begin{align}
&\norm{\wb_{t+1, S, \bm{\xi}} - \wb_{t+1, S, \overline{\bm{\xi}}}} \nonumber\\
&= \norml{\prox{\alpha_t h}\left(\wb_{t, S, \bm{\xi}} - \alpha_t \nabla \ell (\wb_{t, S, \bm{\xi}}; \zb_{\xi_t})\right) - \prox{\alpha_t h}\left(\wb_{t, S, \overline{\bm{\xi}}} + \alpha_t \nabla \ell (\wb_{t, S, \overline{\bm{\xi}}}; \zb_{\xi_t})\right)} \nonumber\\
&\le \frac{1}{1 + \alpha_t \lambda} \norml{\wb_{t, S, \bm{\xi}} - \alpha_t \nabla \ell (\wb_{t, S, \bm{\xi}}; \zb_{\xi_t}) - \wb_{t, S, \overline{\bm{\xi}}} + \alpha_t  \nabla \ell (\wb_{t, S, \overline{\bm{\xi}}}; \zb_{\xi_t})} \nonumber\\
&\le  \frac{1 + \alpha_t L}{1 + \alpha_t \lambda} \norm{\wb_{t, S, \bm{\xi}} - \wb_{t, S, \overline{\bm{\xi}}}} \lesssim \exp(-\alpha_t(\lambda - L)) \norm{\wb_{t, S, \bm{\xi}} - \wb_{t, S, \overline{\bm{\xi}}}}.  \nonumber
\end{align}
Telescoping over $t = t_0, \ldots, T-1$, we further obtain that 
\begin{align}
	\norm{\wb_{T, S, \bm{\xi}} - \wb_{T, S, \overline{\bm{\xi}}}} &\le 2\alpha_{t_0} \sigma \exp\left(-(\lambda - L)\sum_{t = t_0 + 1}^{T-1} \alpha_t \right) \nonumber\\
	&\lesssim \frac{2\sigma c}{(t_0+2)} \exp\left(-(\lambda - L) c \log \frac{T}{(t_0+2)} \right) \nonumber\\
	&=\frac{2\sigma c}{ (t_0+2)^{1 - c(\lambda - L)} T^{c(\lambda - L)}} \nonumber\\
	&\le \frac{2\sigma c}{T^{c(\lambda - L)}}. \nonumber
\end{align}
Thus, from \cref{eq: 30} we obtain that $\rho = \frac{2\sigma^2 c}{T^{c(\lambda - L)}} $. Substituting the expressions of $\beta$ and $\rho$ into \Cref{thm: eliss}, we conclude that with probability at least $1-\delta$
\begin{align}
\Phi(\wb_{T, S}) - \Phi_S(\wb_{T, S}) &\le \frac{4\sigma^2 }{n (\lambda - L)} + \left(\frac{M}{\sqrt{n}} 2\sqrt{n}\frac{2\sigma^2 }{n (\lambda - L)} + \sqrt{2T} \frac{2\sigma^2 c}{T^{c(\lambda - L)}} \right)\sqrt{\log\frac{2}{\delta}} \nonumber\\
&\le \left(\frac{M}{\sqrt{n}} + \frac{4\sigma^2}{\sqrt{n}(\lambda - L)} + \frac{4\sigma^2 c}{T^{c(\lambda - L) - \frac{1}{2}}} \right)\sqrt{\log\frac{2}{\delta}}. \nonumber
\end{align}

\section{Proof of Technical Lemmas for Proximal SGD}\label{sec: support}
For any vector $\gb\in \RR^d$, we define the following quantity:
	\begin{align}
	G^{\alpha}(\wb, \gb):= \tfrac{1}{\alpha}\left(\wb -  \prox{\alpha h}(\wb - \alpha \gb)\right).
	\end{align}
%%%
\begin{lemma}\label{prop: 1}
	Let $h$ be a convex and possibly non-smooth function. Then, the following statements hold.
	\begin{enumerate}[leftmargin=*,topsep=0pt,noitemsep]
		\item  For any $\wb, \gb_1, \gb_2 \in \Omega$, it holds that
		\begin{align}
		\norml{G^{\alpha}(\wb, \gb_1) - G^{\alpha}(\wb, \gb_2)} \le  \norml{\gb_1 - \gb_2}. \nonumber
		\end{align}
		\item If $h$ is $\lambda$ strongly convex, then for all $\wb, \vb\in \Omega$ and $\alpha > 0$, it holds that
		\begin{align}
		\norm{\prox{\alpha h}(\wb) - \prox{\alpha h}(\vb)} \le \tfrac{1}{1+\alpha \lambda} \norm{\wb - \vb}. \nonumber
		\end{align}
	\end{enumerate}
\end{lemma}
\begin{proof}[Proof of \Cref{prop: 1}]
	Consider the first item. By definition, we have
	\begin{align}
	\norml{G^{\alpha}(\wb, \gb_1) - G^{\alpha}(\wb, \gb_2)} 
	&=\frac{1}{\alpha}\norml{\prox{\alpha h}(\wb - \alpha \gb_1) - \prox{\alpha h}(\wb - \alpha \gb_2)} \nonumber\\
	&\le \frac{1}{\alpha} \norml{(\wb - \alpha \gb_1) - (\wb - \alpha \gb_2)} \nonumber\\
	&= \norml{\gb_1 - \gb_2},
	\end{align}
	where the inequality uses the 1-Lipschitz property of the proximal mapping for convex functions.
	
	Next, consider the second item. Recall the resolvent representation \cite{Bauschke_2011} of the proximal mapping for convex functions, i.e.,
	\begin{align}
	\prox{\alpha h}(\wb) = (I + \alpha \nabla h)^{-1} (\wb), \nonumber
	\end{align}
	where $I$ denotes the identity operator. Applying the operator $(I + \alpha \nabla h)$ on both sides of the above equation, we obtain that
	$(I + \alpha \nabla h)(\prox{\alpha h}(\wb)) = \wb$. Thus, we conclude that
	\begin{align}
	\wb - \prox{\alpha h}(\wb) =  \alpha \nabla h(\prox{\alpha h}(\wb)), \nonumber
	\end{align}
	which further implies that
	\begin{align}
	&\inner{[\wb - \prox{\alpha h}(\wb)] - [\vb - \prox{\alpha h}(\vb)]}{\prox{\alpha h}(\wb) - \prox{\alpha h}(\vb)} \nonumber\\
	&\qquad=\alpha \inner{\nabla h(\prox{\alpha h}(\wb)) - \nabla h(\prox{\alpha h}(\vb))}{\prox{\alpha h}(\wb) - \prox{\alpha h}(\vb)}\nonumber\\
	&\qquad \ge \alpha\lambda \norm{\prox{\alpha h}(\wb) - \prox{\alpha h}(\vb)}^2, \nonumber
	\end{align}
	where the last inequality uses the fact that $h$ is $\lambda$-strongly convex. Rearranging the above inequality, we obtain that
	\begin{align}
	&\inner{\wb - \vb}{\prox{\alpha h}(\wb) - \prox{\alpha h}(\vb)} \nonumber\\
	&\qquad \ge (1+\alpha\lambda) \norm{\prox{\alpha h}(\wb) - \prox{\alpha h}(\vb)}^2. \nonumber
	\end{align}
	Applying Cauchy-Swartz inequality on the left hand side, we obtain the desired result. 
\end{proof}

%%%
\begin{lemma}\label{lemma: grad_bound2}
	Let Assumptions \ref{assum: loss}, \ref{assum: var} and \ref{assum: h} hold. Applying the proximal SGD to solve the R-ERM with data set $S$ and choosing $\alpha_t \le \frac{c}{t+2}$ with $0<c<\frac{1}{L}$. Then, it holds that
	\begin{align}
	&\Ebb_{S, \bm{\xi}} \left[\norm{\nabla \ell (\wb_{t, S}; \zb_{1})}\right] \le  \sqrt{2L\Phi(\wb_{0}) + 2\Ebb_{S}[\nu_S^2] \log t}. \nonumber
	\end{align}
\end{lemma}
\begin{proof}[Proof of \Cref{lemma: grad_bound2}]
	The proof is based on the technical tools developed in \cite{Ghadimi_2016} for analyzing the optimization path of the proximal SGD. 
%	For any vector $\gb\in \RR^d$, we define the following quantity:
%	\begin{align}
%	G^{\alpha}(\wb, \gb):= \tfrac{1}{\alpha}\left(\wb -  \prox{\alpha h}(\wb - \alpha \gb)\right).
%	\end{align}
	Under the assumptions of the lemma, we first recall the following result from [Lemma 1, \cite{Ghadimi_2016}]: For any $\wb \in \Omega, \gb\in \RR^d$, it holds that
	\begin{align}
	\inner{\gb}{G^{\alpha}(\wb, \gb)} \ge \norm{G^{\alpha}(\wb, \gb)}^2 + \frac{1}{\alpha}\left(h(\prox{\alpha h}(\wb - \alpha \gb)) - h(\wb) \right). \nonumber
	\end{align}
	Denoting $\gb_{t,S} =\nabla \ell (\wb_{t, S}; \zb_{\xi_t})$ as the stochastic gradient sampled at iteration $t$ and setting $\wb = \wb_{t, S}, \gb = \gb_{t,S}$ in the above inequality, we obtain that
	\begin{align}
	\inner{\gb_{t,S}}{G^{\alpha_t}(\wb_{t, S}, \gb_{t,S})} \ge \norm{G^{\alpha}(\wb_{t, S}, \gb_{t,S})}^2 + \frac{1}{\alpha_t}\left(h(\wb_{t+1, S}) - h(\wb_{t,S}) \right). \label{eq: 36}
	\end{align}
	On the other hand, using \cref{eq: 33} and non-negativity of $h$, we obtain
	\begin{align}
	\Ebb_{\bm{\xi}, S} \norm{\nabla \ell (\wb_{t, S}; \zb_{1})} 
	&\le \sqrt{2L}  \sqrt{\Ebb_{\bm{\xi}, S} f_S(\wb_{t, S})} \le \sqrt{2L}  \sqrt{\Ebb_{\bm{\xi}, S} \Phi_S(\wb_{t, S})}. \label{eq: 39}
	\end{align}
	Next, consider a fixed $S$, by the smoothness of $\ell$ we obtain
	\begin{align}
	&f_S(\wb_{t+1, S}) - f_S(\wb_{t, S}) \nonumber\\
	&\le \inner{\wb_{t+1, S} - \wb_{t, S}}{\nabla f_S(\wb_{t, S})} + \frac{L}{2} \norm{\wb_{t+1, S} - \wb_{t, S}}^2 \nonumber\\
	&= \inner{-\alpha_t G^{\alpha_t}(\wb_{t, S}, \gb_{t,S}) }{\nabla f_S(\wb_{t, S})} + \frac{L\alpha_t^2}{2} \norml{G^{\alpha_t}(\wb_{t, S}, \gb_{t,S})}^2 \nonumber \\
	&= -\alpha_t \inner{G^{\alpha_t}(\wb_{t, S}, \gb_{t,S}) }{\gb_{t,S}} -\alpha_t \inner{G^{\alpha_t}(\wb_{t, S}, \gb_{t,S})}{\nabla f_S(\wb_{t,S}) - \gb_{t,S}} + \frac{L\alpha_t^2}{2} \norml{G^{\alpha_t}(\wb_{t, S}, \gb_{t,S})}^2 \nonumber \\
	&= -\alpha_t \inner{G^{\alpha_t}(\wb_{t, S}, \gb_{t,S})}{\gb_{t,S}} -\alpha_t \inner{G^{\alpha_t}(\wb_{t, S}, \nabla f_S(\wb_{t,S})) }{\nabla f_S(\wb_{t,S}) - \gb_{t,S}} + \frac{L\alpha_t^2}{2} \norml{G^{\alpha_t}(\wb_{t, S}, \gb_{t,S})}^2 \nonumber \\
	&\quad + \alpha_t \inner{ G^{\alpha_t}(\wb_{t, S}, \nabla f_S(\wb_{t,S})) - G^{\alpha_t}(\wb_{t, S}, \gb_{t,S})}{\nabla f_S(\wb_{t,S}) - \gb_{t,S}}.
	\end{align}
	Now combining with \cref{eq: 36} and rearranging, we obtain that 
	\begin{align}
	&\Phi_S(\wb_{t+1, S}) - \Phi_S(\wb_{t, S}) \nonumber\\
	&\le \left(\frac{L\alpha_t^2}{2} - \alpha_t\right) \norml{G^{\alpha_t}(\wb_{t, S}, \gb_{t,S})}^2
	-\alpha_t \inner{G^{\alpha_t}(\wb_{t, S}, \nabla f_S(\wb_{t,S})) }{\nabla f_S(\wb_{t,S}) - \gb_{t,S}}  \nonumber \\
	&\quad + \alpha_t \inner{ G^{\alpha_t}(\wb_{t, S}, \nabla f_S(\wb_{t,S})) - G^{\alpha_t}(\wb_{t, S}, \gb_{t,S}) }{\nabla f_S(\wb_{t,S}) - \gb_{t,S}} \nonumber\\
	&\le \left(\frac{L\alpha_t^2}{2} - \alpha_t\right) \norml{G^{\alpha_t}(\wb_{t, S}, \gb_{t,S})}^2
	-\alpha_t \inner{G^{\alpha_t}(\wb_{t, S}, \nabla f_S(\wb_{t,S}))}{\nabla f_S(\wb_{t,S}) - \gb_{t,S}}  \nonumber \\
	&\quad + \alpha_t \norm{G^{\alpha_t}(\wb_{t, S}, \nabla f_S(\wb_{t,S})) - G^{\alpha_t}(\wb_{t, S}, \gb_{t,S})} \norm{\nabla f_S(\wb_{t,S}) - \gb_{t,S}} \nonumber\\
	&\le \left(\frac{L\alpha_t^2}{2} - \alpha_t\right) \norml{G^{\alpha_t}(\wb_{t, S}, \gb_{t,S})}^2
	-\alpha_t \inner{G^{\alpha_t}(\wb_{t, S}, \nabla f_S(\wb_{t,S}))}{\nabla f_S(\wb_{t,S}) - \gb_{t,S}}+ \alpha_t \norm{\nabla f_S(\wb_{t,S}) - \gb_{t,S}}^2, \nonumber
	\end{align}
	where the last line uses item 1 of \Cref{prop: 1}.
	Conditioning on $\wb_{t, S}$, and taking expectation with respect to $\bm{\xi}$, we further obtain from the above inequality that
	\begin{align}
	\Ebb_{\bm{\xi}} &[\Phi_S(\wb_{t+1, S}) - \Phi_S(\wb_{t, S})~|~\wb_{t, S}] \nonumber \\
	& \le \left(\frac{L\alpha_t^2}{2} - \alpha_t\right) \Ebb_{\bm{\xi}} \left[\norml{G^{\alpha_t}(\wb_{t, S}, \gb_{t,S})}^2~|~\wb_{t, S}\right] + \alpha_t \Ebb_{\bm{\xi}} \left[\norm{\nabla f_S(\wb_{t,S}) - \gb_{t,S}}^2~|~\wb_{t, S}\right]. \nonumber
	\end{align}
	Further taking expectation with respect to the randomness of $\wb_{t, S}$ and $S$, telescoping the above inequality over $0, \ldots, t-1$ and noting that $\frac{L\alpha_t^2}{2} < \alpha_t$, we obtain that 
	\begin{align}
	\Ebb_{\bm{\xi}, S} \left[\Phi_S(\wb_{t, S})\right] 
	&\le \Ebb_{S} \Phi_S(\wb_{0}) + \sum_{t'=0}^{t-1} \frac{c\Ebb_{S}[\nu_S^2]}{t'+2} \nonumber\\
	&\le \Phi(\wb_{0}) + c\Ebb_{S}[\nu_S^2] \log t, \nonumber
	\end{align}
	where we have used the bound for the variance of the stochastic gradients.
	Substituting the above expression into \cref{eq: 39} and note that $cL<1$, we obtain the desired result.
\end{proof}